\documentclass[journal,compsoc]{IEEEtran}

\usepackage{graphicx}
\usepackage{amsthm}
\usepackage{color}
\usepackage{epstopdf}
\usepackage{subfigure}
\usepackage{amsmath}
\usepackage{amssymb}
\usepackage{multirow}
\usepackage[bookmarks=true]{hyperref}
\usepackage{bookmark}
\usepackage{algorithm}
\usepackage{algorithmicx, algpseudocode}
\usepackage{enumitem}
\usepackage{mathtools}
\usepackage{thmtools}
\usepackage{slashbox}
\usepackage{caption}
\declaretheorem{theorem}
\newtheorem*{problem}{Problem}

\makeatletter %
\newcommand\mysubsection{\@startsection{paragraph}{4}{\z@}{3\p@ \@plus \p@}{-5\p@}{\normalsize\bfseries}}

\def \myalg {RSDNE}
\def \mydeepalg {RECT}

\makeatletter
\def\@seccntformatinl#1{\csname the#1dis\endcsname\hskip 1em\relax}
\makeatother

\begin{document}
\title{Network Embedding with Completely-imbalanced Labels}
\author{Zheng~Wang,
        Xiaojun~Ye,
        Chaokun~Wang, \IEEEmembership{Member,~IEEE},
        Jian~Cui,
        and Philip~S.~Yu,~\IEEEmembership{Fellow,~IEEE}
\IEEEcompsocitemizethanks{
\IEEEcompsocthanksitem Z. Wang and J. Cui are with the Department of Computer Science and Technology, University of Science and Technology Beijing, Beijing, China.\protect\\
E-mail: wangzheng@ustb.edu.cn, g20178653@xs.ustb.edu.cn.

\IEEEcompsocthanksitem X. Ye and C. Wang are with the School of Software, Tsinghua University, Beijing, China.\protect\\
E-mail: \{yexj, chaokun\}@tsinghua.edu.cn.

*Corresponding author: Chaokun Wang.
\IEEEcompsocthanksitem P.S. Yu is with the Department of Computer Science, University of Illinois at Chicago IL 60607. E-mail: psyu@cs.uic.edu.
}
\thanks{Accepted by IEEE TKDE (DOI 10.1109/TKDE.2020.2971490).}
}

\markboth{IEEE TRANSACTIONS ON KNOWLEDGE AND DATA ENGINEERING}%
{Shell \MakeLowercase{\textit{et al.}}: Bare Demo of IEEEtran.cls for Computer Society Journals}

\IEEEcompsoctitleabstractindextext{%
\begin{abstract}\label{section_abstract}
Network embedding, aiming to project a network into a low-dimensional space, is increasingly becoming a focus of network research.
Semi-supervised network embedding takes advantage of labeled data, and has shown promising performance.
However, existing semi-supervised methods would get unappealing results in the \textbf{completely-imbalanced} label setting where some classes
have no labeled nodes at all.
To alleviate this, we propose two novel semi-supervised network embedding methods.
The first one is a shallow method named RSDNE.
Specifically, to benefit from the completely-imbalanced labels, \myalg\ guarantees both intra-class similarity and inter-class dissimilarity in an approximate way.
The other method is \mydeepalg\ which is a new class of graph neural networks.
Different from RSDNE, to benefit from the completely-imbalanced labels, \mydeepalg\ explores the class-semantic knowledge.
This enables \mydeepalg\ to handle networks with node features and multi-label setting.
Experimental results on several real-world datasets demonstrate the superiority of the proposed methods. Code is available at \url{https://github.com/zhengwang100/RECT}.

\end{abstract}

}

\maketitle

\section{Introduction}\label{sect_intro}
\IEEEPARstart{N}{etwork} analysis~\cite{wang2016causality}~\cite{wang2019ecoqug}~\cite{wang2018parallel}~\cite{li2019scaling} is a hot research topic in various scientific areas like social science, computer science, biology and physics.
Many algorithmic tools for network analysis heavily rely on network representation which is traditionally represented by the adjacency matrix.
However, this straightforward representation not only lacks of  representative power but also suffers from the data sparsity issue~\cite{buhlmann2011statistics}.

Recently, learning dense and low-dimensional vectors as representations for networks has aroused considerable
research interest in network analysis.
It has been shown that the learned representations could benefit many network analysis tasks, such as node classification~\cite{perozzi2014deepwalk}, link prediction~\cite{grover2016node2vec}~\cite{wang2018deepdirect} and network visualization~\cite{tang2016visualizing}.
Commonly, learning network representation is also known as network embedding~\cite{moyano2017learning}.
The learned low-dimensional vectors are called node embeddings (or representations).

\begin{figure}[!t]
\centering
\subfigure[Unsupervised network embedding]{
\label{fig_ne_unsupervised}
    \includegraphics[width=0.5\textwidth]{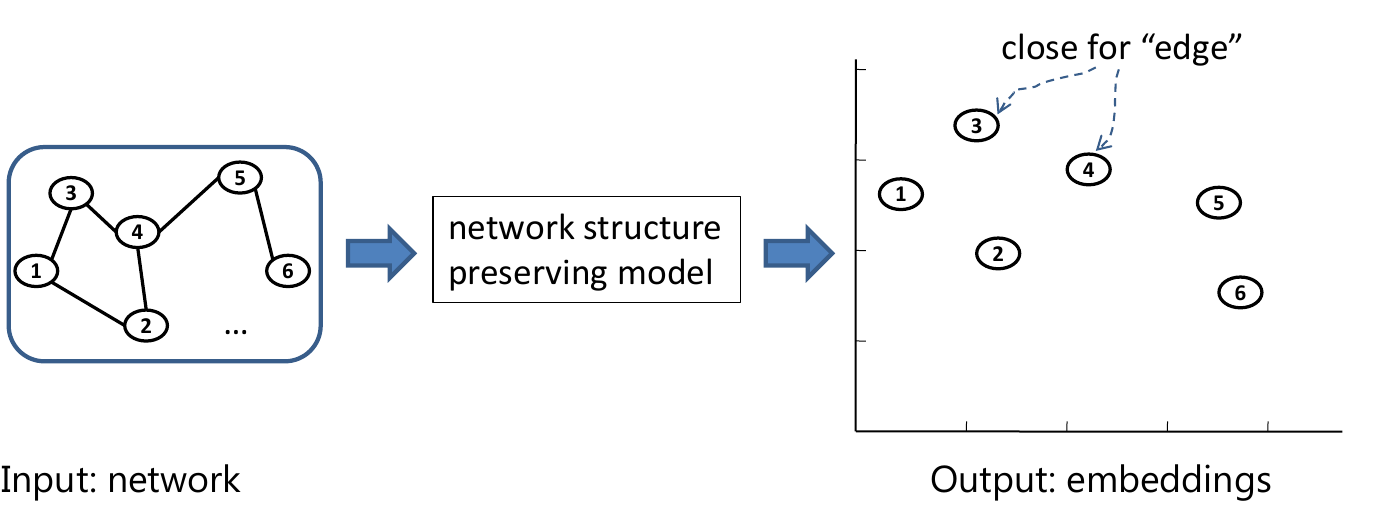}
}
\subfigure[Semi-supervised network embedding]{
\label{fig_ne_supervised}
    \includegraphics[width=0.5\textwidth]{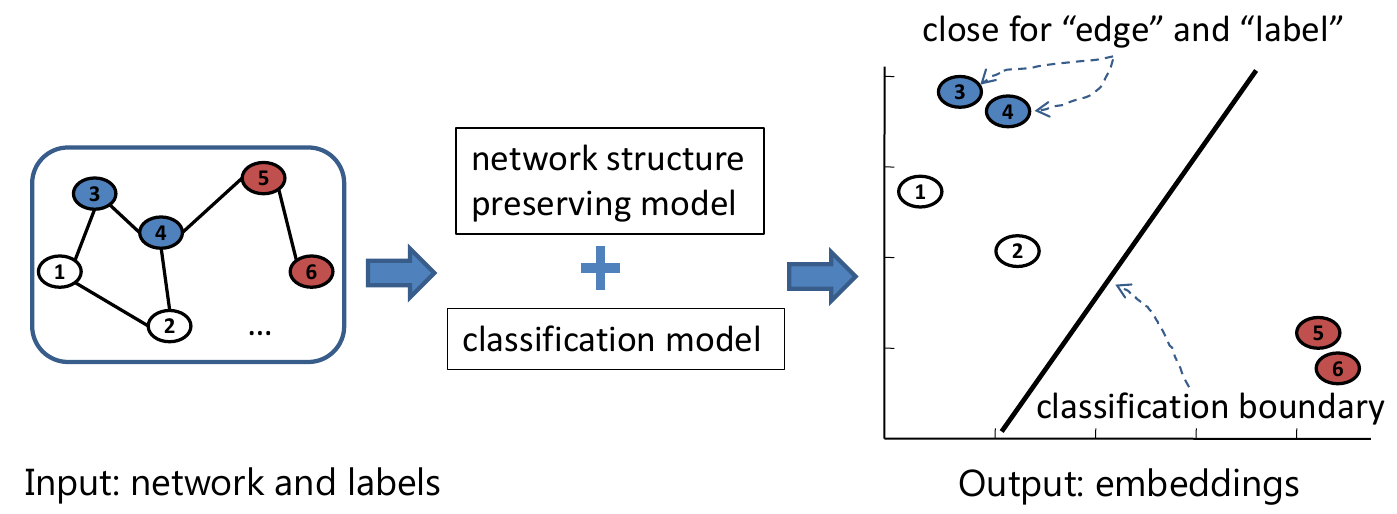}
}
\caption{Frameworks of existing unsupervised and semi-supervised network embedding methods.}
\label{fig_ne}
\end{figure}

One basic requirement of network embedding is to preserve the inherent network structure in the embedding space, as illustrated in Fig.~\ref{fig_ne_unsupervised}.
Early studies, like IsoMap~\cite{tenenbaum2000global} and LLE~\cite{roweis2000nonlinear}, ensure the embedding similarity among linked nodes.
Now, more research activities focus on preserving the unobserved but legitimate links in the network.
For example, DeepWalk~\cite{perozzi2014deepwalk} exploits the node co-occurring relationships in the truncated random walks over a network.
LINE~\cite{tang2015line}~\cite{wang2017equivalence} considers both the first-order and second-order proximities of a network.
Unlike the above two shallow methods, SDNE~\cite{wang2016structural}, a class of graph neural networks (GNNs)~\cite{gori2005new}~\cite{scarselli2008graph}, uses multiple layers of non-linear functions to model these two proximities.

Semi-supervised network embedding methods, which take advantage of labeled data, have shown promising performance.
Typical semi-supervised shallow methods include LSHM~\cite{jacob2014learning}, LDE~\cite{wang2016linked}, and MMDW~\cite{tu2016max}.
Typical semi-supervised GNNs are GCN~\cite{kipf2017semi}, GAT~\cite{velickovic2018graph} and APPNP~\cite{klicpera2019predict}.
As illustrated in Fig.~\ref{fig_ne_supervised}, in these methods, a classification model (e.g., SVM~\cite{hearst1998support} and Cross-entropy~\cite{de2005tutorial}) will be learned to inject label information.
Intuitively, in the embedding space, the learned classification model would reduce the distance between same labeled nodes and enlarge the distance between different labeled nodes.
Influenced by this, the embedding results therefore become more discriminative and have shown state-of-the-art performance.
\begin{figure}[!t]
\centering
    \includegraphics[width=0.5\textwidth]{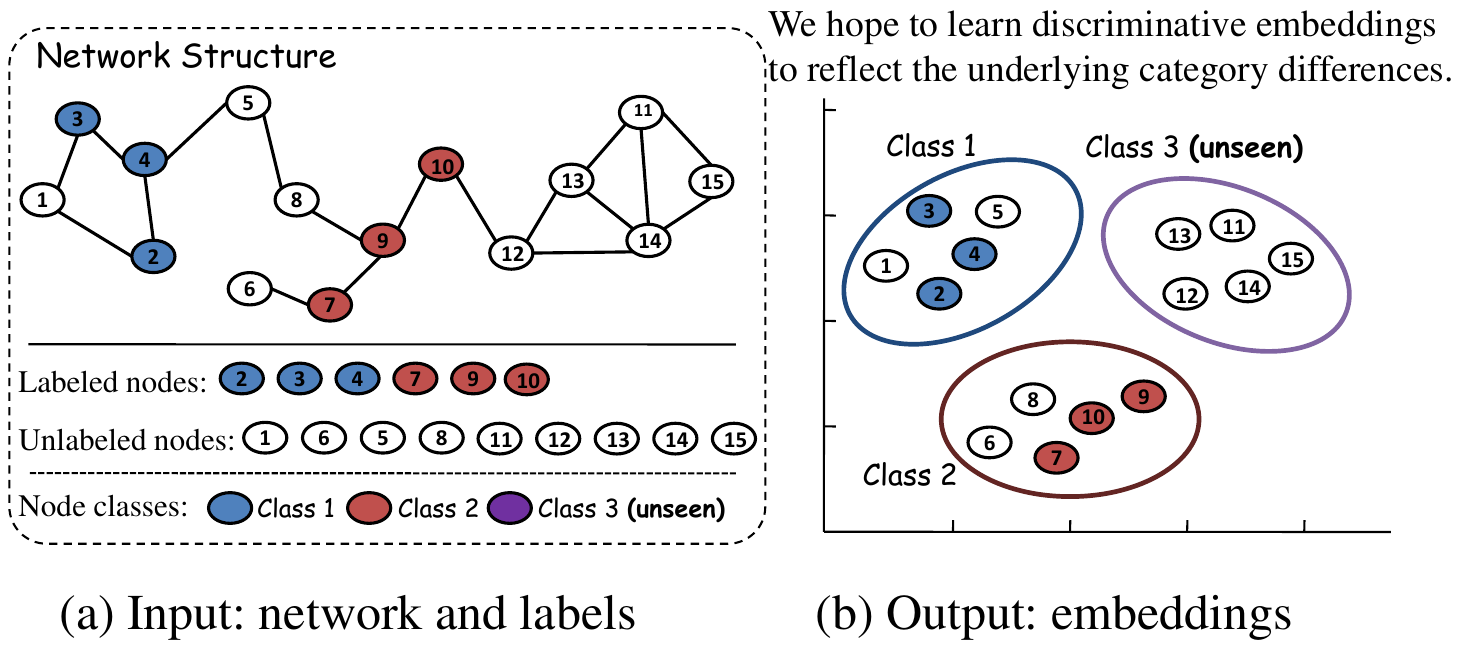}
\caption{
Illustration of the network embedding with completely-imbalanced labels.
This toy network actually contains three classes of nodes, but only two classes provide labeled nodes, i.e., blue and red nodes.
The remaining nodes (including all the nodes of Class 3) are unlabeled.
}
\label{fig_example}
\end{figure}

\subsection{Problem}
Most semi-supervised network embedding methods~\cite{jacob2014learning}~\cite{wang2016linked}~\cite{tu2016max} assume the labeled data is generally balanced, i.e., every class has at least one labeled node.
In this paper, we consider a more challenging scenario in which some classes have no labeled nodes at all (shown in Fig.~\ref{fig_example}), i.e., the \textbf{completely-imbalanced} case.
This problem can be formulated as follows:
\begin{problem}[Network embedding with completely-imbalanced labels]\label{define_problem}
Given a network $\mathcal{G}=(\mathcal{V}, A, \mathcal{C}, \mathcal{C}^{s})$ where $\mathcal{V}$ is the set of $n$ nodes, $A \in \mathbb{R}^{n\times n}$ is the adjacency matrix, $\mathcal{C}$ is the whole node class label set, and $\mathcal{C}^{s} {\subset} \mathcal{C}$ is the observed label set, our goal is to learn a continuous low-dimensional vector $u_{i} \in \mathbb{R}^{d}$ ($d {\ll} n$) for each node $v_i$, such that nodes close to each other in the network structure and with the similar class labels are close in the embedding space.
\end{problem}

This problem deserves special attention for two reasons.
Firstly, it has many practical applications.
For example, considering Wikipedia which can be seen as a set of linked web pages on various topics~\cite{de2017inducing}, it is difficult to collect labeled samples for every topic exactly and not miss any one.
Secondly, and more importantly, without considering this issue, traditional semi-supervised methods would yield unappealing results.
To verify this, we carry out an experiment on Citeseer dataset~\cite{mccallum2000automating}, in which the nodes from unseen classes are excluded from the labeled data.
We test two typical semi-supervised methods (i.e., a shallow method LSHM and a GNN method GCN) on node classification task.
As shown in Table~\ref{tab_micro-f1_intro}, their performance declines noticeably compared with their counterparts trained with the balanced labels.
This decline might be caused by the classification models used in these methods, since general classifiers are very likely to get biased results on imbalanced data~\cite{he2009learning}.
We refer to Sections~\ref{sect_analysis} and~\ref{sect_expriment} for more detailed discussion.

\subsection{Contribution}
To address this problem, in this paper, we first present a novel shallow method termed \myalg.
The basic idea is to guarantee both intra-class similarity and inter-class dissimilarity in an approximate way, so as to benefit from completely-imbalanced labels.
Specifically, we relax the intra-class similarity requirement by allowing the same labeled nodes to lie on the same manifold in the embedding space.
On the other hand, we approximate the inter-class dissimilarity requirement by removing the known connections between the nodes with different labels.
As such, our method can reasonably guarantee these two requirements and also avoid the biased results.
We further formalize these approximations into a unified embedding framework, and give an efficient learning algorithm.

\begin{table}[!t]
\setlength{\tabcolsep}{4.5pt}
\scriptsize
\caption{Classification performance on Citeseer. Here: we use $\mathcal{M}$($b$) and $\mathcal{M}$(-$t$) to denote the method $\mathcal{M}$ using the balanced and completely-imbalanced labeled data with $t$ unseen classes, respectively.}
\centering
\begin{tabular}{cc||lll|lll}
\hline
& &  \multicolumn{3}{c|}{Accuracy} &\multicolumn{3}{c}{Relative Accuracy Decline} \\
\hline
\hline
\multicolumn{2}{c||}{\backslashbox{\emph{\textbf{Method}}}{\emph{\textbf{Label}}}} & 10\% & 30\% & 50\% & 10\% & 30\% & 50\%\\
\hline
\multirow{3}*{LSHM}
&LSHM(b) &0.5007 &0.6178 &0.6711 &- &- &-  \\
&LSHM(-1) &0.4258 &0.5887 &0.6455 &0.1496$\downarrow$	&0.0471$\downarrow$	&0.0382$\downarrow$  \\
&LSHM(-2) &0.4253 &0.5504 &0.6027 &0.1506$\downarrow$	&0.1091$\downarrow$	&0.1019$\downarrow$ \\
\hline
\multirow{3}*{GCN}
&GCN(b) &0.7198 &0.7473 &0.7628 &- &- &- \\
&GCN(-1) &0.6572 &0.6937 &0.7064 &0.0870$\downarrow$	&0.0717$\downarrow$	&0.0739$\downarrow$ \\
&GCN(-2) &0.4761 &0.5085 &0.5159 &0.3386$\downarrow$	&0.3196$\downarrow$	&0.3237$\downarrow$ \\
\hline
\end{tabular}
\label{tab_micro-f1_intro}
\end{table}

To leverage the power of deep neural networks~\cite{lecun2015deep}, we further propose \mydeepalg, a new class of GNNs.
Comparing to \myalg, \mydeepalg\ can further leverage node features and deal with the multi-label case~\cite{herrera2016multilabel}.
In particular, to utilize the completely-imbalanced labels, unlike RSDNE nor traditional GNNs, \mydeepalg\ adopts a novel objective function which explores the class-semantic knowledge.
This is motivated by the recent success of Zero-Shot Learning (ZSL)~\cite{xian2018zero}, which has demonstrated the ability of recognizing unseen objects via introducing class-semantic descriptions.
In addition, unlike the traditional ZSL methods, the class-semantic descriptions used in \mydeepalg\ do not rely on human annotations or any third-party resources, making \mydeepalg\ well suited for practical applications.

In summary, our main contributions are as follows:
\begin{enumerate}
\setlength\itemsep{0em}
  \item We study the problem of network embedding with completely-imbalanced labels. To our best knowledge, little work has
addressed this problem.
  \item We propose an effective shallow method named \myalg\ which can learn discriminative embeddings by approximately guaranteeing both intra-class similarity and inter-class dissimilarity.
  \item We propose \mydeepalg , a new class of graph neural networks. Comparing to \myalg, \mydeepalg\ can further handle networks with node features and multi-label setting.
  \item We conduct extensive experiments on five real-world datasets in both completely-imbalanced setting and balanced setting to demonstrate the superiority of our methods.
\end{enumerate}
In addition, it is worth highlighting that in the balanced label setting, our methods could still achieve comparable performance to state-of-the-art semi-supervised methods, although our methods are not specially designed for this setting.
Therefore, our methods would be favorably demanded by the scenario where the quality of labels cannot be guaranteed.

The remainder of this paper is organized as follows.
We review some related work in Section~\ref{sect_related}.
In Section~\ref{sect_shallow_method}, we elaborate our shallow method RSDNE with details.
In Section~\ref{sect_deep_method}, we introduce the proposed GNN method \mydeepalg.
Section~\ref{sect_analysis} discusses the rationality of our methods, and further analyzes the relationship between the existing methods and ours.
Section~\ref{sect_expriment} reports experimental results.
Section~\ref{section_conclusion} concludes this paper.

\section{Related Work}\label{sect_related}

\subsection{Semi-supervised Network Embedding}
The goal of semi-supervised network embedding is to learn the representations of both labeled and unlabeled nodes.
Existing shallow methods mainly share the similar idea, that is, to jointly train a network structure preserving model and a class classification model.
For example, LDE~\cite{wang2016linked} considers the first-order proximity~\cite{tang2015line} of the network and jointly trains a 1-nearest neighbor classification model~\cite{mucherino2009k}.
Semi-supervised GNNs also will train a classification model but implicitly preserve the network structure information.
In particular, most GNNs (like GCN~\cite{kipf2017semi}, GAT~\cite{velickovic2018graph} and APPNP~\cite{klicpera2019predict}) iteratively perform feature aggregations based on the network structure~\cite{battaglia2018relational}.
We refer readers to a comprehensive survey~\cite{wu2019comprehensive} for more discussions.

However, these methods all assume the labeled data is generally balanced (i.e., label information covers all classes), otherwise would get unappealing results.
In practice, the quality of labeled data is hard to guarantee.
Therefore, to enhance the applicability, we investigate network embedding in the completely-imbalanced label setting.

\subsection{Imbalanced Data Learning}
A training dataset is called imbalanced if at least one of the classes are represented by significantly less number of instances than the others.
The imbalanced data are pervasively existed in multiple domains ranging from the physical world to social networks, and to make proper use of such data is always a pivotal challenge~\cite{Cai2019TiiTaxi}~\cite{Cai2018TNSECPSS}.
This topic has been identified in several vital research areas, such as classification~\cite{sun2007cost}, clustering~\cite{yen2009cluster}, and data streams~\cite{yan2017framework}.
We refer to~\cite{he2009learning} and~\cite{krawczyk2016learning} for a comprehensive survey.
However, in the area of network embedding, little previous work considers the imbalanced problem, not to mention the completely-imbalanced problem~\cite{moyano2017learning}.

\subsection{Zero-Shot Learning} \label{subsect_zsl}
ZSL~\cite{lampert2009learning}~\cite{farhadi2009describing}, which is recently a hot research topic in computer vision, aims to recognize the objects from unseen classes.
To achieve this goal, it leverages some high-level semantic descriptions (also called as attributes) shared between both seen and unseen classes.
For example, we can define some attributes like ``wing'', ``climb'' or ``tail'' for animals.
Then we can train attribute recognizers using images and attribute information from seen classes.
After that, given an image from unseen classes, we can infer its attributes.
By comparing the difference between the inferred attributes and each unseen classes' attributes, the final output is given based on the score.
Generally, attributes are human annotated, which needs lots of human efforts.
Another more practical way is to use word embeddings generated by word2vec tools~\cite{mikolov2013efficient} trained with large-scale general text database.
Despite of this, attributes collection still heavily relies on third-party resources, limiting the use of ZSL methods in practical applications.


Till now, although various ZSL methods have been proposed~\cite{geng2018recent}, all these methods are limited to classification or prediction scenario~\cite{wang2019zero}.
To our best knowledge, there is little reported work considering the unseen classes in the network embedding problem.
This problem can be seen as a new variation of ZSL or, more properly, as the problem of zero-shot graph embedding (ZGE) that aims to learn effective node representations for both seen and unseen classes.

\section{The Proposed Shallow Method: RSDNE}\label{sect_shallow_method}
In this section, we first introduce a network structure preserving model.
Then, we present our method with another two objective terms for completely-imbalanced labels.
Finally, we give an efficient optimization algorithm.

\subsection{Modeling Network Structure with DeepWalk}
To capture the topological structure of a network, DeepWalk performs random walks over a network to get node sequences.
By regarding each node sequence $\omega =\{v_{1}, ..., v_{|\omega|}\}$ as a word sequence, it adopts the well-known language model Skip-Gram~\cite{mikolov2013efficient} to  maximize the likelihood of the surrounding nodes given the current node $v_{i}$ for all random walks $\omega \in \Omega$:
\begin{equation}
\label{eq_deepwalk_original}
\sum_{\omega\in \Omega}[\frac{1}{|\omega|} \sum_{i=1}^{|\omega|} \sum_{ -r \le j \le r} \log Pr(v_{i+j}|v_{i})]
\end{equation}
where $r$ is the radius of the surrounding window, and the probability $Pr(v_{j}|v_{i})$ is obtained via the softmax:
\begin{equation}
\label{eq_softmax}
Pr(v_{j}|v_{i}) = \frac{exp(u_{j} \cdot u_{i})}{\sum_{t\in \mathcal{V}} exp(u_{t} \cdot u_{i}) }
\end{equation}
where $u_{i}$ is the representation vector of node $v_{i}$, and $\cdot$ is the inner product between vectors.

Yang et al.~\cite{yang2015network} has proved that DeepWalk actually factorizes a matrix $M$ whose entry $M_{ij}$ is formalized as:
\begin{equation}
\label{eq_deepwalk}
M_{ij} = \log \left. [e_{i}(\bar{A}+\bar{A}^2+ \cdots + \bar{A}^t)] \middle / t \right.
\end{equation}
where $\bar{A}$ is the transition matrix which can be seen as a row normalized network adjacency matrix, and $e_{i}$ denotes an indicator vector whose $i$-th entry is 1 and the others are all 0.
To balance speed and accuracy, ~\cite{yang2015network} finally factorized the matrix $M {=} (\bar{A}{+}\bar{A}^2)/2$ instead, since sparse matrix multiplication can be easily parallelized and efficiently calculated~\cite{polok2014fast}.

More formally, the matrix factorization model of DeepWalk aims to find a (node embedding) matrix $U \in \mathbb{R}^{n \times d}$ and a (context embedding) matrix $H \in \mathbb{R}^{d \times n}$ via solving the following optimization problem:
\begin{equation}
\label{eq_mf_basic}
\begin{aligned}
\min_{U,H}  &\hspace{0.5em} \mathcal{J}_{DW} {=} \left \| M -UH \right \|^{2}_{F} + \lambda ( \left \| U \right \|^{2}_{F} + \left \| H \right \|^{2}_{F} )
\end{aligned}
\end{equation}
where $\lambda$ is the regularization parameter to avoid overfitting.
In this paper, we adopt this model (i.e., Eq.~\ref{eq_mf_basic}) as our basic network structure preserving model.

\subsection{Modeling Intra-class Similarity}
In this completely-imbalanced setting, the labeled nodes all come from the seen classes.
Intuitively, we should ensure the \emph{intra-class similarity}, i.e., the nodes sharing the same label should be close to each other in the embedding space.
To satisfy this, traditional semi-supervised methods employ various classifiers to reduce the intra-class embedding variance.
However, this would yield unappealing results with completely-imbalanced labels (shown in Table~\ref{tab_micro-f1_intro}).

To alleviate this, we relax this similarity requirement by allowing the same labeled nodes to lie on the same manifold, i.e., a topological space which can be Euclidean only locally~\cite{roweis2000nonlinear}.
Although the underlying manifold is unknown, we can build a sparse adjacency graph to approximate it~\cite{belkin2007convergence}.
In other words, each labeled node only needs to be close to $k$ ($k {\ll} n$, and $k{=}$5 in our experiments) same labeled nodes.
However, we do not know how to select the best $k$ nodes, since the optimal node alignments in the new embedding space is unknown.
A simple solution is to randomly select $k$ same labeled nodes, which may not be optimal.

In this paper, we solve this problem in an adaptive way.
For notational convenience, for a labeled node $v_i$, we call the selected $k$ nodes as $v_i$'s \emph{intra-class neighbors}.
Suppose we use $S {\in} \{0,1\}^{n\times n}$ to denote the intra-class neighbor relationship among nodes, i.e., $S_{ij}{=}1$ when node $v_j$ is the intra-class neighbor of node $v_i$, otherwise $S_{ij}{=}0$.
Mathematically, $S$ can be obtained by solving the following optimization problem:
\begin{equation}
\label{eq_intra_cost}
\begin{aligned}
\min_{U,S}  &\hspace{0.5em} \mathcal{J}_{intra} {=} \frac{1}{2} \sum_{ i,j=1}^{n} \left \| u_{i} - u_{j} \right \|^{2}_{F} S_{ij} \\
\mathrm{s.t.}   &\hspace{0.5em} \forall i \in \mathcal{L}, s_{i}'\textbf{1} = k, \ S_{ii}=0 \\
           &\hspace{0.5em}  \forall i,j\in \mathcal{L}, S_{ij} \in \{0,1\}, \ \mathrm{if} \ \mathcal{C}^{s}_{i} = \mathcal{C}^{s}_{j} \\
            &\hspace{0.5em} \forall i,j, S_{ij}=0, \ \mathrm{if} \ i \notin \mathcal{L} \ \mathrm{or}\ \mathcal{C}^{s}_{i} \ne \mathcal{C}^{s}_{j} \\
\end{aligned}
\end{equation}
where $\mathcal{L}$ is the labeled node set, and $s_{i} {\in} R^{n \times 1}$ is a vector with the $j$-th element as $S_{ij}$ (i.e., $s_{i}'$, the transpose of $s_{i}$, is the row vector of matrix $S$), and \textbf{1} denotes a column vector with all entries equal to one,
and $\mathcal{C}^{s}_{i}$ and $\mathcal{C}^{s}_{j}$ are the (seen) class labels of node $v_i$ and $v_j$ respectively.
In this paper, $(\cdot)'$ stands for the transpose.

\subsection{Modeling Inter-class Dissimilarity}
Although Eq.~\ref{eq_intra_cost} models the similarity within the same class, it neglects the \emph{inter-class dissimilarity}, i.e., the nodes with different labels should be far away from each other in the embedding space.
Traditional semi-supervised methods employ different classification models to enlarge the inter-class embedding variance.
Nevertheless, this would yield unappealing results with completely-imbalanced labels (shown in Table~\ref{tab_micro-f1_intro}).

To alleviate this, we approximate this dissimilarity requirement by removing the known connections between the nodes with different labels.
Since we adopt the matrix form of DeepWalk (i.e., matrix $M$ in Eq.~\ref{eq_mf_basic}) to model the connections among nodes, this approximation leads to the following optimization problem:
\begin{equation}
\label{eq_inter_cost}
\begin{aligned}
\min_{U}  &\hspace{0.5em} \mathcal{J}_{inter} {=} \frac{1}{2} \sum_{ i,j = 1}^{n} \left \| u_{i} - u_{j} \right \|^{2}_{F} W_{ij}
\end{aligned}
\end{equation}
where $W$ is a weighted matrix whose element $W_{ij}{=}0$ when labeled nodes $v_i$ and $v_j$ belong to different categories, otherwise $W_{ij}=M_{ij}$.

\subsection{The Unified Model: \myalg}
With modeling the network structure (Eq.~\ref{eq_mf_basic}), intra-class similarity (Eq.~\ref{eq_intra_cost}) and inter-class dissimilarity (Eq.~\ref{eq_inter_cost}), the proposed method is to solve the following optimization problem:
\begin{equation}
\label{eq_final_cost}
\begin{aligned}
\min_{U,H,S}  &\hspace{0.5em} \mathcal{J} {=} \mathcal{J}_{DW}  + \alpha( \mathcal{J}_{intra} + \mathcal{J}_{inter}) \\
\mathrm{s.t.}   &\hspace{0.5em} \forall i \in \mathcal{L}, s_{i}'\textbf{1} = k, \ S_{ii}=0 \\
           &\hspace{0.5em}  \forall i,j\in \mathcal{L}, S_{ij} \in \{0,1\}, \ \mathrm{if} \ \mathcal{C}^{s}_{i} = \mathcal{C}^{s}_{j} \\
            &\hspace{0.5em} \forall i,j, S_{ij}=0, \ \mathrm{if} \ i \notin \mathcal{L} \ \mathrm{or}\ \mathcal{C}^{s}_{i} \ne \mathcal{C}^{s}_{j}
\end{aligned}
\end{equation}
where $\alpha$ is a balancing parameter.
Since both the relaxed similarity and dissimilarity requirements of labels have been considered, we call the proposed method as \textbf{R}elaxed \textbf{S}imilarity and \textbf{D}issimilarity \textbf{N}etwork \textbf{E}mbedding (\myalg).

\mysubsection*{A Light Version of \myalg}
For each labeled node $v_i$, to identify its optimal $k$ intra-class neighbors, \myalg\ needs to consider all the nodes which have the same label with $v_i$.
This would become inefficient when more labeled data is available (some theoretical analysis can be found in Section~\ref{sect_time_complexity}).
Therefore, we give a light version of \myalg\ (denoted as \myalg$^{*}$).
The idea is that: for a labeled node $v_i$, at the beginning, we can randomly select $\bar{k}$ ($k{<}\bar{k} {\ll} n$) same labeled nodes to gather $v_i$'s intra-class neighbor candidate set $\mathcal{O}_{i}$.
Based on this idea, this light version \myalg$^{*}$\ is to solve the following optimization problem:
\begin{equation}
\label{eq_final_cost_light}
\begin{aligned}
\min_{U,H,S}  &\hspace{0.5em} \mathcal{J} {=} \mathcal{J}_{DW}  + \alpha( \mathcal{J}_{intra} + \mathcal{J}_{inter}) \\
\mathrm{s.t.}   &\hspace{0.5em} \forall i \in \mathcal{L}, s_{i}'\textbf{1} = k, \ S_{ii}=0 \\
           &\hspace{0.5em}  \forall i\in \mathcal{L}, j\in \mathcal{O}_{i}, S_{ij} \in \{0,1\} \\
            &\hspace{0.5em} \forall i,j, S_{ij}=0, \ \mathrm{if} \ i \notin \mathcal{L} \ \mathrm{or}\ \mathcal{C}^{s}_{i} \ne \mathcal{C}^{s}_{j} \\
\end{aligned}
\end{equation}

\subsection{Optimization}
\subsubsection{Optimization for \myalg}
The objective function in Eq.~\ref{eq_final_cost} is a standard quadratic programming problem with 0/1 constraints, which might be difficult to solve by the conventional optimization tools.
In this study, we propose an efficient alternative optimization strategy for this problem.

\mysubsection*{Update $U$ As Given $H$ and $S$}
When $S$ is fixed, the objective function in Eq.~\ref{eq_intra_cost} can be rewritten as $Tr(U'L_{s}U)$, where $L_{s} = D_{s} - (S+S')/2 $ and $D_{s}$ is a diagonal matrix whose $i$-th diagonal element is $\sum_{j}(S_{ij} + S_{ji})/2$.
Similarly, the objective function in Eq.~\ref{eq_inter_cost} can be rewritten as $Tr(U'L_{w}U)$ where $L_{w} = D_{w} - (W+W')/2 $ and $D_{w}$ is a diagonal matrix whose $i$-th diagonal element is $\sum_{j}(W_{ij} + W_{ji})/2$.
As such, when $H$ and $S$ are fixed, problem (\ref{eq_final_cost}) becomes:
\begin{equation}
\label{eq_final_cost_U}
\begin{aligned}
\min_{U}  &\hspace{0.2em} \mathcal{J}_{U} {=} \left \| M {-} UH \right \|^{2}_{F} {+} \alpha ( Tr(U'L_{s}U) {+} Tr(U'L_{w}U) )  {+} \lambda \left \| U \right \|^{2}_{F}
\end{aligned}
\end{equation}
\noindent The derivative of $\mathcal{J}_{U}$ w.r.t. $U$ is:
\begin{equation}
\label{eq_dev_U}
\begin{aligned}
\frac{\partial \mathcal{J}_{U}}{\partial U} = 2(-MH'+UHH' {+} \alpha (L_{s} {+} L_{w})U {+} \lambda U )
\end{aligned}
\end{equation}

\mysubsection*{Update $H$ As Given $U$ and $S$}
When $U$ and $S$ are fixed, problem (\ref{eq_final_cost}) becomes:
\begin{equation}
\label{eq_final_cost_V}
\begin{aligned}
\min_{H}  &\hspace{0.5em} \mathcal{J}_{H} {=}\left \| M -UH \right \|^{2}_{F} + \lambda  \left \| H \right \|^{2}_{F}
\end{aligned}
\end{equation}

\noindent The derivative of $\mathcal{J}_{H}$ w.r.t. $H$ is:
\begin{equation}
\label{eq_dev_V}
\begin{aligned}
\frac{\partial \mathcal{J}_{H}}{\partial H} &=2(-U'M + U'UH + \lambda H )
\end{aligned}
\end{equation}

\mysubsection*{Update $S$ As Given $U$ and $H$}
When $U$ and $H$ are fixed, problem (\ref{eq_final_cost}) becomes:
\begin{equation}
\label{eq_final_cost_S}
\begin{aligned}
\min_{S}  &\hspace{0.5em} \mathcal{J}_{S} = \frac{\alpha}{2} \sum_{ i,j=1}^{n} \left \| u_{i} - u_{j} \right \|^{2}_{F} S_{ij} \\
\mathrm{s.t.}   &\hspace{0.5em} \forall i \in \mathcal{L}, s_{i}'\textbf{1} = k, \ S_{ii}=0 \\
           &\hspace{0.5em}  \forall i,j\in \mathcal{L}, S_{ij} \in \{0,1\}, \ \mathrm{if} \ \mathcal{C}^{s}_{i} = \mathcal{C}^{s}_{j} \\
            &\hspace{0.5em} \forall i,j, S_{ij}=0, \ \mathrm{if} \ i \notin \mathcal{L} \ \mathrm{or}\ \mathcal{C}^{s}_{i} \ne \mathcal{C}^{s}_{j}
\end{aligned}
\end{equation}
\noindent As problem (\ref{eq_final_cost_S}) is independent between different $i$, we can deal with the following problem individually for each labeled node $v_i$\footnote{For an unlabeled node $v_i$, the solution is $s'_{i}=0$.}:
\begin{equation}
\label{eq_sub_problem_S}
\begin{aligned}
\min_{ s_{i}, i \in \mathcal{L}}  &\hspace{0.5em} \sum_{j=1}^{n} \left \| u_{i} - u_{j} \right \|^{2}_{F} S_{ij} \\
\mathrm{s.t.}   &\hspace{0.5em} s_{i}'\textbf{1} = k,\ S_{ii}=0   \\
            &\hspace{0.5em} \forall j, S_{ij}=0, \ \mathrm{if} \ j \notin \mathcal{L} \\
            &\hspace{0.5em}  \forall j\in \mathcal{L}, S_{ij} \in \{0,1\}, \ \mathrm{if} \ \mathcal{C}^{s}_{i} = \mathcal{C}^{s}_{j}
\end{aligned}
\end{equation}
\noindent The optimal solution to problem (\ref{eq_sub_problem_S}) is (proved in Section~\ref{sect_proof}):
\begin{align}
\label{eq_solution_S}
 S_{ij} &=
  \begin{cases}
   1,        & \text{if } v_j \in \mathcal{N}_{kc}(v_i){;} \\
   0,        & \text{otherwise.}
  \end{cases}
\end{align}
where set $\mathcal{N}_{kc}(v_i)$ contains the top-$k$ nearest and same labeled nodes to $v_i$ in the current calculated embedding space.

For clarity, we summarize the complete \myalg\ algorithm for network embedding in Alg.~\ref{alg}.

\begin{algorithm}[!tb]
\caption{\myalg}
\label{alg}
\begin{algorithmic}[1]
\Require Matrix form of DeepWalk $M$, label information, learning rate $\eta$, and parameters $\alpha$ and $\lambda$ ;
\Ensure The learned network node embedding result $U$;
    \State Initialize $U$, $H$ and $S$;
    \Repeat
        \State Update $U$ by $U = U - \eta \frac{\mathcal{J}_{U}}{\partial U} $; 
        \State Update $H$ by $H = H - \eta \frac{\mathcal{J}_{H}}{\partial H} $; 
        \State Update $S$ by solving problem (\ref{eq_final_cost_S}) ;
        \State Change the learning rate $\eta$ according to some rules, such as Armijo~\cite{bertsekas1997nonlinear};
    \Until{Convergence or a certain number of iterations;} \\
    \Return $U$.
\end{algorithmic}
\end{algorithm}

\subsubsection{Optimization for \myalg$^{*}$}
The optimization approach for \myalg$^{*}$ is almost the same as Alg.~\ref{alg}.
The only difference is that: when updating $S$ as given $U$ and $H$, for each labeled node $v_i$, we only need to sort the nodes in (it's intra-class neighbor candidate set) $\mathcal{O}_{i}$ to get the top-$k$ nearest and same labeled neighbors, so as to get the optimal solution of $S$.

\section{The Proposed GNN Method: \mydeepalg}\label{sect_deep_method}
It is inappropriate to directly adopt the objective function of RSDNE (Eq.~\ref{eq_final_cost} or Eq.~\ref{eq_final_cost_light}) into traditional neural networks (like multilayer perceptron) which are not suited for graph-structured\footnote{In the rest of paper, we use the term ``graph'' to refer to the linked data structures such as social or biological networks, so as to avoid ambiguity with neural network terminology.} data.
Moreover, simultaneously optimizing multiple objective terms is a challenging engineering task, and usually results in a degenerate solution~\cite{glorot2010understanding}.
In this section, we first give a brief introduction to GNN, and then propose a novel effective and easy-to-implement GNN method.

\subsection{Preliminaries: Graph Neural Network}
GNN~\cite{scarselli2008graph} is a type of neural network model for graph-structured data.
Generally, GNN models are dynamic models where the hidden representations of all nodes evolve over layers.
Given a graph with the adjacent matrix $A$, at the $t$-th hidden layer, the representation $z_{v_i}^{t}$ for node $v_i$ is commonly updated as follows:
\begin{equation}
\label{eq_gnn}
\begin{aligned}
b_{v_i}^{t} & = \mathcal{F}_{b}(\{ z_{v_{j}}^{t}|v_{j} \in \Psi_{v_i}\}) \\
z_{v_i}^{t+1} & = \mathcal{F}_{z}(\{b_{v_{i}}^{t}, z_{v_{i}}^{t}\})
\end{aligned}
\end{equation}
where $b_{v_i}^{t}$ is a vector indicating the aggregation of messages that node $v_i$ receives from its neighbors $\Psi_{v_i}$.
Function $\mathcal{F}_{b}$ is a message calculating function, and $\mathcal{F}_{z}$ is a hidden state update function.
Similar to the common neural networks, $\mathcal{F}_{b}$ and $\mathcal{F}_{z}$ are feed-forward neural layers.
By specifying these two functional layers, we can get various GNN variants, like Graph convolutional network (GCN)~\cite{kipf2017semi} and Graph attention network (GAT)~\cite{velickovic2018graph}.
\begin{figure}[]
\centering
    \includegraphics[width=0.5\textwidth]{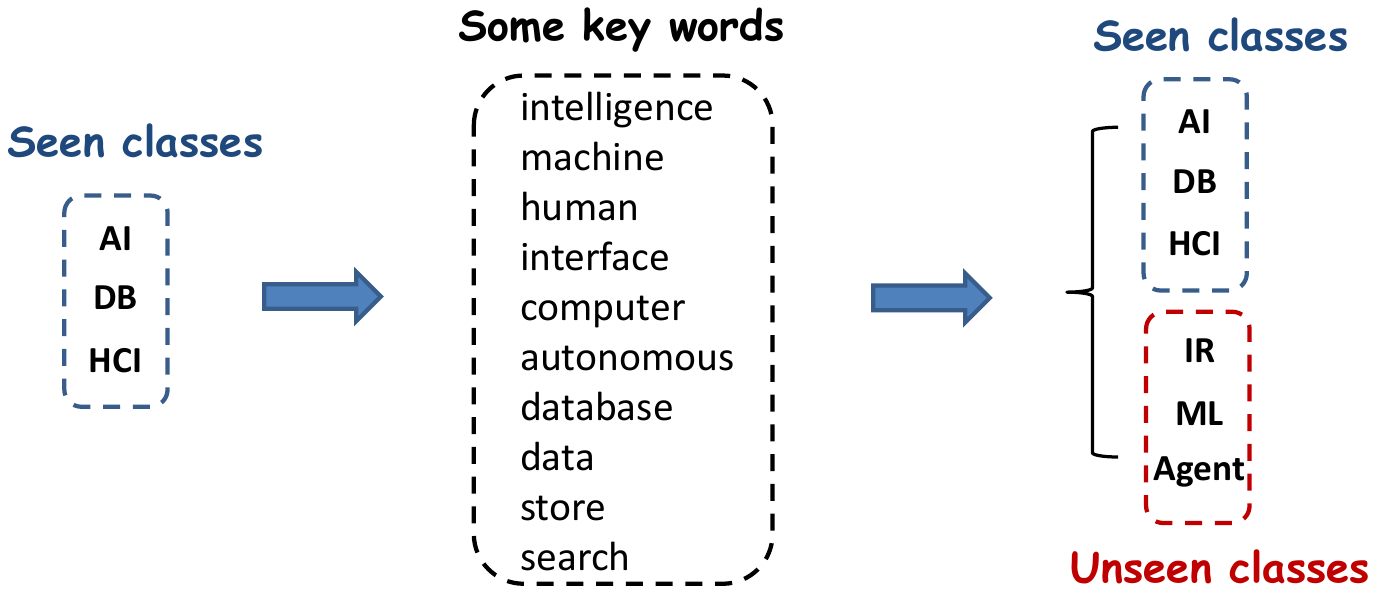}
\caption{Some words sampled from the documents of three seen classes (i.e., AI, DB, and HCI) in Citeseer.}
\label{fig_semantic}
\end{figure}

\begin{figure*}[ht]
\centering
    \includegraphics[width=0.75\textwidth]{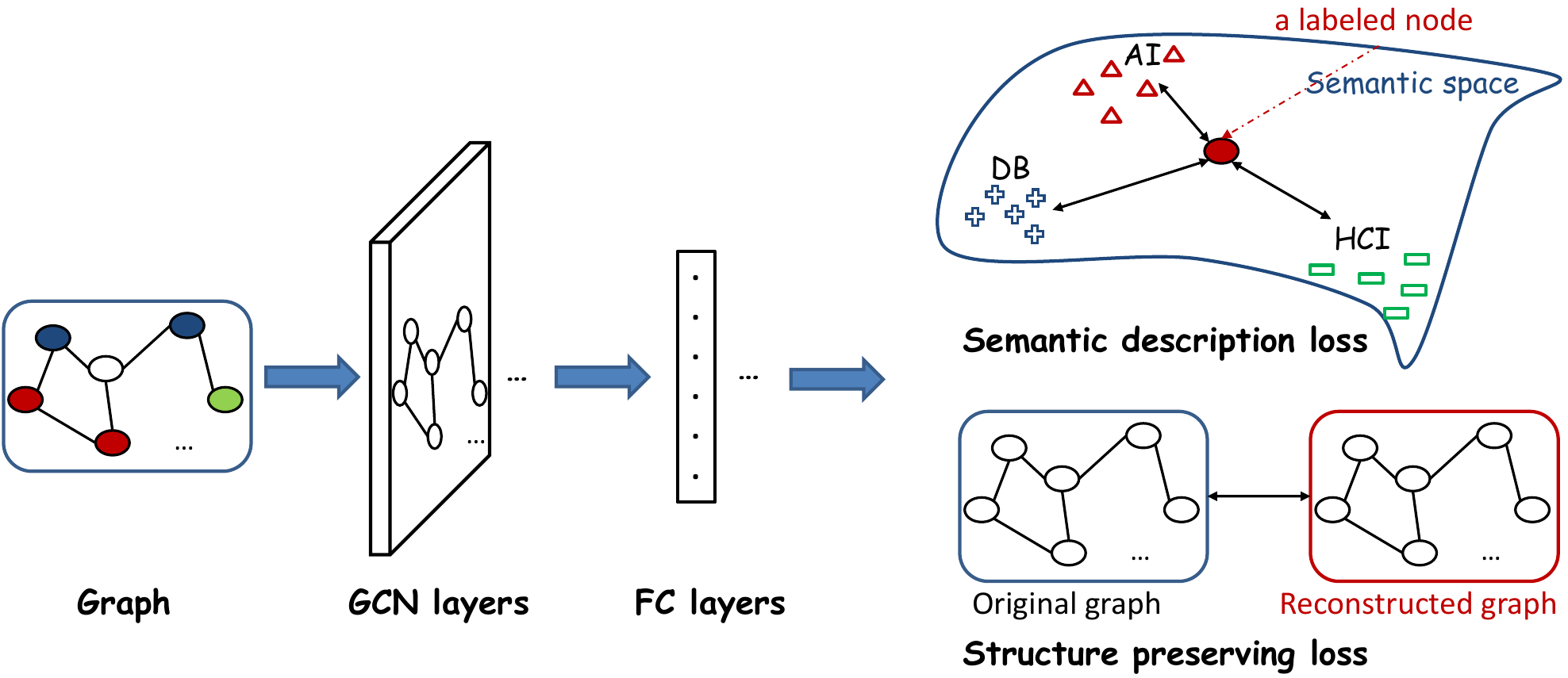}
\caption{Architecture overview of \mydeepalg.}
\label{fig_dear_overview}
\end{figure*}

To inject label information, GNNs usually end up with a softmax layer to train a classification model.
Once the training of GNNs is completed, the outputs of any hidden layers can be adopt as the final graph embedding results.
However, as shown in Table~\ref{tab_micro-f1_intro}, this kind of methods will yield unappealing results in completely-imbalanced label setting.
The fundamental cause is that the known supervised information only reflects the knowledge of seen classes but ignores that of unseen classes.
Therefore, in the completely-imbalanced setting, the key issue is: how to deduce the supervised information, which contains both the knowledge of seen and unseen classes, from the limited labeled nodes of seen classes.

\subsection{Deduce Supervised Information for Unseen Classes}
\subsubsection{Observation}
The recent success of ZSL demonstrates that the capacity of inferring semantic descriptions (also known as attributes) makes it possible to categorize unseen objects.
Generally, the attributes are human annotated or provided by third-party resources (like the word embeddings learned from large-scale general text database), limiting the use of ZSL methods.
In addition, the quality of attributes can be a source of problems in practical applications.

For graph embedding, we propose to obtain class-semantic descriptions in a more practical manner.
To show its feasibility, we continue to use the citation graph Citeseer~\cite{mccallum2000automating} as an example.
Figure~\ref{fig_semantic} shows some words sampled from the documents of AI, DB, and HCI classes in this dataset.
Interestingly, these words also reflect some knowledge of other three (unseen) research areas (i.e., IR, ML and Agent).
For example, IR's key words (like ``human'', ``search'' and ``data'') also show up in the documents of the (seen) research areas DB and HCI.
This observation inspires us to generate class-semantic descriptions directly from the original node features.

\subsubsection{Generate Class-semantic Descriptions Automatically}
Let matrix $X \in \mathbb{R}^{n\times m}$ denote the feature matrix, where $x_{i}\in \mathbb{R}^m$ (the $i$-th row of $X$) is the corresponding $m$-dimensional feature vector of node $v_{i}$.
To obtain the semantic descriptions for a seen class $c$, we can leverage a readout function $\mathcal{R}$, and use it to summarize a class-semantic description vector (denoted as $\hat{y}_{c}$) from the labeled nodes, i.e., $\hat{y}_{c} = \mathcal{R} (\{ x_{i}|\forall_{i}\ \mathcal{C}^{s}_{i}=c \})$.

For those graphs without node features, we can treat the rows of adjacency matrix as node features.
Intuitively, each node can be seen as a word, and all nodes construct a dictionary.

\subsection{The Proposed Model: \mydeepalg}
The architecture of \mydeepalg\ is illustrated in Fig.~\ref{fig_dear_overview}.
In detail, we first adopt GCN layers to explore graph structure information.
After propagating through all CGN layers, fully-connected (FC) layers are used to project the outputs of GCN layers into a semantic vector space, in which the loss is computed.
Here, we use FC layers rather than GCN layers, because we hope to improve the robustness of the learned embeddings by satisfying our objective function without explicitly using the graph structure knowledge.

Our loss function consists of two parts.
The first one is a prediction loss in the semantic space, i.e., the loss between the predicted and the actual class-semantic description vectors:
\begin{equation}
\label{eq_semantic_loss}
\mathcal{J}_{semantic} = \sum_{i\in \mathcal{L}} loss(\hat{y'}_{C_{i}^{s}}, \hat{y}_{C_{i}^{s}})
\end{equation}
where $\hat{y'}_{C_{i}^{s}}$ and $\hat{y}_{C_{i}^{s}}$ are the predicted and the actual class-semantic vector of the labeled node $v_i$ respectively, and $loss(\cdot, \cdot)$ is a sample-wise loss function.
By using this loss, our method can capture the class-semantic knowledge, making the learned graph embeddings reflect the supervised information of both seen and unseen classes.

The second is a graph structure preserving loss.
Unlike GCN or other semi-supervised GNNs, we still propose to explicitly preserve the graph structure knowledge.
This is because the above loss (Eq.~\ref{eq_semantic_loss}) actually indirectly preserves the label discrimination, which would reduce the discrimination of learned embeddings (especially in the seen classes).
For simplicity, here we follow the similar idea of our shallow method RSDNE.
Specifically, the learned node embeddings $U$ (i.e., the outputs of the last layer) should minimize:
\begin{equation}
\label{eq_deep_graph_loss}
\mathcal{J}_{graph\_neural} = loss(M, UU')
\end{equation}

To learn powerful embeddings by considering both parts, a simple and effective way we find in practice is to train the model which considers
these two parts separately and then concatenate the embeddings trained by the two parts for each node.
A more principled way to combine these two loss parts is to jointly train the objective functions Eq.~\ref{eq_semantic_loss} and Eq.~\ref{eq_deep_graph_loss}, which we leave as future work.

For clarity, we summarize this method in Alg.~\ref{alg_deep}.
We refer this method as \textbf{RE}laxed G\textbf{C}N Ne\textbf{T}work (RECT), as it utilizes GCN model and relaxes the original label discrimination by preserving class-semantic knowledge.

\begin{algorithm}[!tb]
\caption{\mydeepalg}
\label{alg_deep}
\begin{algorithmic}[1]
\Require Graph information (i.e., $A$ and $X$), label information $\mathcal{L}$;
\Ensure The learned node embedding result $U$;
    \State Summarize the class-semantic descriptions of seen classes through the readout function $\mathcal{R}$;
    \State Obtain the embedding result $U^{(1)}$ by optimizing \mydeepalg\ with the objective function Eq.~\ref{eq_semantic_loss} ;
    \State Obtain the embedding result $U^{(2)}$ by optimizing \mydeepalg\ with the objective function Eq.~\ref{eq_deep_graph_loss} ;
    \State Obtain the final embedding result $U$ by concatenating the normalized $U^{(1)}$ and $U^{(2)}$ ; \\
    \Return $U$.
\end{algorithmic}
\end{algorithm}

\section{Algorithm Analysis}\label{sect_analysis}
\subsection{Optimization Algorithm Solving Problem (\ref{eq_sub_problem_S})}\label{sect_proof}
\begin{theorem} \label{pro_alg}
The optimal solution of problem (\ref{eq_sub_problem_S}) is Eq.~\ref{eq_solution_S}.
\end{theorem}

\begin{proof}\label{proof_alg}
By contradiction, suppose a labeled node $v_i$ has gotten its optimal intra-class neighbor set $\mathcal{N}_{kc}$ which contains a node $v_p$ not in $v_i$'s top-$k$ nearest and same labeled nodes.
As such, there must exist a node $v_q \notin \mathcal{N}_{kc}$ which is one of $v_i$'s top-$k$ nearest and same labeled nodes.
Then, we get $ \| u_{i} - u_{p} \|_{F}^{2} > \| u_{i} - u_{q} \|_{F}^{2}$.
Considering our minimization problem (i.e., Eq.~\ref{eq_sub_problem_S}), this inequation leads:
\begin{equation}
\label{eq_proof}
\begin{aligned}
\sum_{j\in \mathcal{N}_{kc}} \left \| u_{i} - u_{j} \right \|^{2}_{F} > \sum_{j\in \{\mathcal{N}_{kc} + v_q\} \setminus v_p } \left \| u_{i} - u_{j} \right \|^{2}_{F}
\end{aligned}
\end{equation}
This indicates that $\{\mathcal{N}_{kc} {+} v_q\}{\setminus} v_p$ is a better optimal solution than $\mathcal{N}_{kc}$, a contradiction.
\end{proof}

\subsection{Time Complexity Analysis}\label{sect_time_complexity}
\subsubsection*{Complexity of RSDNE}
Following~\cite{rao2015collaborative}, the time complexity of Alg.~\ref{alg} is as below.
The complexity for updating $U$ is $O(nnz(M)d + d^{2}n + nnz(L)d)$, where $nnz(\cdot)$ is the number of non-zeros of a matrix.
The complexity for updating $H$ is $O(nnz(M)d + d^{2}n)$.
The complexity for updating $S$ is  $O( |\mathcal{C}^{s}|\ell^{2} \log \ell)$, where $\ell = rn |\mathcal{C}^{s}|/|\mathcal{C}|$ is the average number of labeled nodes per class, and $r$ is the label rate.
As $\ell $ is linear with $n$ and $nnz(L)$ is linear with $nnz(M)$, the overall complexity of \myalg\ is $O( \tau ( nnz(M)d {+} n^{2} \log n) ) $, where $\tau$ is the number of iterations to converge.

\mysubsection*{Complexity of \myalg$^{*}$}
For the light version, i.e., \myalg$^{*}$, the complexity of updating $S$ becomes $O( |\mathcal{C}^{s}|\bar{k}^{2} \log \bar{k})$, and all others remain the same.
Hence, as $\bar{k} {\ll} n$, the overall complexity becomes  $O( \tau ( nnz(M)d {+} d^{2}n ) ) $.
As our method typically converges fast ($\tau \leq$ 15 in our experiments) and $d \ll n$, the complexity of \myalg$^{*}$ is linear to $nnz(M)$ and node number $n$.

\mysubsection*{Complexity of RECT}
First of all, the time cost of the GCN layer is linear in the number of graph edges~\cite{kipf2017semi}.
Specifically, the time complexity is $O( m|\mathcal{E}||d^{h}||\mathcal{C}^{stc}|)$, where $|\mathcal{E}|$ is the edge number and $|d^{h}|$ is the hidden layer dimension size and $|\mathcal{C}^{stc}|$ is the dimension of class-semantic description.
The complexity of calculating Eq.~\ref{eq_semantic_loss} is $O( n|\mathcal{C}^{stc}|)$.
The complexity of calculating Eq.~\ref{eq_deep_graph_loss} is $O( dn^2)$.
Therefore, the total complexity of RECT is $O( m|\mathcal{E}||d^{h}||\mathcal{C}^{stc}|+n|\mathcal{C}^{stc}|+dn^2)$.
Note we can directly reduce this complexity by adopting other graph structure preserving objectives, like the objective of DeepWalk (i.e., Eq.~\ref{eq_deepwalk_original}).
Then, the total complexity will reduce to $O( m|\mathcal{E}||d^{h}||\mathcal{C}^{stc}|+n|\mathcal{C}^{stc}|+ d n \log n)$, indicating the similar complexity as DeepWalk and GCN.

\subsection{The Proposed Methods\ v.s. Traditional Semi-supervised Methods}\label{sect_compare}
\subsubsection{Traditional Semi-supervised Methods}
To benefit from the discriminative information (e.g., class labels), the most effective and widely used strategy is to guarantee both the intra-class similarity and inter-class dissimilarity in the embedding space~\cite{lin2006inter,kan2016multi}.
For this purpose, traditional semi-supervised graph embedding methods reduce the intra-class embedding variance and enlarge the inter-class embedding variance by optimizing various classification models.
However, as the unseen class nodes are (partly) linked with the seen class ones (i.e., seen and unseen class nodes are correlated), only optimizing over the seen classes is suboptimal for the whole graph.

In those shallow methods (like LSHM), this suboptimal strategy would impose lots of strict constraints (like the ``close-to'' constraints between same labeled nodes) only on seen classes, which may seriously mislead the jointly trained graph structure preserving model and finally lead to very poor results.
Similarly in those GNNs which implicitly preserve the graph structure, this suboptimal strategy would also mislead the used message aggregation mechanism and finally lead to very poor results.

\begin{table*}[!t]
\caption{The Statistics of Datasets.}
\centering
\begin{tabular}{lrrrrr}
\hline
Name     &Citeseer  & Cora     &Wiki & PPI &Blogcatalog \\
\hline
Type & Citation graph &Citation graph     &Hyperlink graph &Biological graph &Social graph \\
Nodes        &3,312  & 2,708        &2,405   &3,890   &10,312    \\
Edges        &4,732  & 5,429      &17,981    &76,584 &333,983     \\
Classes      &6    & 7         &17   &50   &39 \\
Features     &3,703 & 1,433        &4,973    & -  & -\\
Multi-label  & No   & No     &No      &YES &YES \\
\hline
\end{tabular}
\label{tab_dataset}
\end{table*}

\subsubsection{The Relation of RSDNE}
RSDNE actually relaxes these above-mentioned strict constraints in shallow methods.
We show in the following that the intra-class similarity loss defined in~\cite{lin2006inter} is a special case of our Eq.~\ref{eq_intra_cost}.
This equivalence also explains the rationale of our method.
\begin{theorem} \label{pro_intra}
In each seen class $c$, let $k_{c}$ and $l_{c}$ denote the intra-class neighbor number and the labeled node number in this class, respectively.
For each labeled class $c$, if we enlarge $k_{c}$ to $l_{c}$, Eq.~\ref{eq_intra_cost} is equivalent to the intra-class similarity equation.
\end{theorem}

\begin{proof}\label{proof_pro_intra}
The intra-class similarity function in ~\cite{lin2006inter} is defined to minimize:
\begin{equation}
\label{eq_intra_cost_citation}
\sum_{i=1}^{n}\sum_{j:\mathcal{C}^{s}_{i} = \mathcal{C}^{s}_{j}} \left \| u_{i} - u_{j} \right \|^{2}_{F}
\end{equation}

In each seen class $c$, if we set $k_{c}=l_c$, Eq.~\ref{eq_intra_cost} actually minimizes:
\begin{equation}
\label{eq_intra_cost_full}
\begin{aligned}
 &\hspace{2.5em}  \sum_{i,j=1}^{n} \left \| u_{i} - u_{j} \right \|^{2}_{F} S_{ij} \\
\mathrm{s.t.} &\hspace{0.5em}  \forall i,j\in \mathcal{L}, S_{ij} \in \{0,1\}, \ \mathrm{if} \ \mathcal{C}^{s}_{i} = \mathcal{C}^{s}_{j} \\
            &\hspace{0.5em} \forall i,j, S_{ij}=0, \ \mathrm{if} \ i \notin \mathcal{L} \ \mathrm{or}\ \mathcal{C}^{s}_{i} \ne \mathcal{C}^{s}_{j} \\
\end{aligned}
\end{equation}
As Eq.~\ref{eq_intra_cost_full} equals Eq.~\ref{eq_intra_cost_citation}, the conclusion is proved.
\end{proof}

Similarly, in RSDNE, the objective function part formulated in Eq.~\ref{eq_inter_cost} actually relaxes the classical inter-class dissimilarity.
Specifically, in Eq.~\ref{eq_inter_cost}, $W_{ij}$ measures the similarity score between node $v_i$ and $v_j$.
For two different labeled nodes $v_i$ and $v_j$, setting $W_{ij}$ to a large negative number reflects the intuition of inter-class dissimilarity.
In sum, these two relaxation strategies not only reasonably guarantee both intra-class similarity and inter-class dissimilarity, but also avoid misleading the jointly trained graph structure preserving model.
Consequently, RSDNE would benefit from completely-imbalanced labels, which is further verified in our experiments.

\subsubsection{The Relation of \mydeepalg}\label{sect_compare_deep}
\mydeepalg\ and traditional GNNs share the similar neural network architecture.
The fundamental difference is the objective function.
Traditional GNNs preserve the class-label discrimination.
\mydeepalg\ aims to preserve the class-semantic knowledge.
As shown in related ZSL studies, class-semantic knowledge enables the knowledge transfer from seen classes to unseen classes, making the learned embeddings reflect the supervised knowledge of both seen and unseen classes.
Intuitively, \mydeepalg\ can also be seen as a relaxation of the class-label discrimination by preserving the class-semantic knowledge.

\section{Experiments}\label{sect_expriment}
\mysubsection*{Datasets}
We conduct our experiments on five real-world graphs, whose statistics are listed in Table~\ref{tab_dataset}.
Citeseer~\cite{mccallum2000automating} and Cora~\cite{mccallum2000automating} are citation graphs whose nodes are articles, edges are citations, and labels are research areas.
Wiki~\cite{sen2008collective} is a set of Wikipedia pages.
In this dataset, nodes are web pages, edges are hyperlinks among them, and labels are topics.
PPI~\cite{grover2016node2vec} is a biological graph dataset, and Blogcatalog~\cite{tang2009relational} is a social graph dataset.
Their labels are biological states and user interests, respectively.
Unlike the previous ones, the nodes in these two graphs may have multiple labels.
In addition, these two graphs do not have node features, and we use the rows of their adjacency matrices as node features.

\mysubsection*{Baseline Methods}
We compare the proposed methods against the following baselines:

\begin{table*}[!t]
\setlength{\tabcolsep}{4pt} 
\scriptsize
\caption{Micro-F1 scores on classification tasks.
The best result is marked in bold.
In the case of no node features, the best result is marked with underline.}
\centering
\begin{tabular}{cc||l|ll|lll|ll|ll|ll|l|ll}
\hline
\multicolumn{2}{c||}{\emph{\textbf{Information}}} & \multicolumn{1}{c|}{$X$} & \multicolumn{2}{c|}{$A$} & \multicolumn{3}{c|}{$A$, $\mathcal{L}$} &\multicolumn{2}{c|}{$A$, $X$} & \multicolumn{2}{c|}{$A$, $X$, $\mathcal{L}$} &  \multicolumn{2}{c|}{$A$, $\mathcal{L}$} &$A$, $X$ & \multicolumn{2}{c}{$A$, $X$, $\mathcal{L}$}\\
\hline
\multicolumn{2}{c||}{\backslashbox{\emph{\textbf{Data}}}{\emph{\textbf{Method}}}} & NodeFeats & MFDW & LINE & LSHM & LDE  &MMDW & TADW &DGI &GCN &APPNP & RSDNE & RSDNE* & \mydeepalg-N & \mydeepalg-L & \mydeepalg\\
\hline
\hline
\multirow{3}*{Citeseer}
&10\% &0.6535 &0.4810 &0.4448 &0.4253 &0.4515 &0.5141 &0.6844 &0.7014 &0.5640 &0.5944 &0.5395 &\underline{0.5426} &0.6975 &0.6601 &\textbf{0.7083} \\
&30\% &0.7006 &0.5793 &0.4959 &0.5504 &0.5224 &0.6020 &0.7187 &0.7293 &0.5889 &0.6274 &\underline{0.6313} &0.6271 &0.7301 &0.7154 &\textbf{0.7403}  \\
&50\% &0.7161 &0.6096 &0.5084 &0.6027 &0.5805 &0.6278 &0.7276 &0.7377 &0.5995 &0.6356 &\underline{0.6741} &0.6683 &0.7359 &0.7294 &\textbf{0.7475}  \\
\hline
\multirow{3}*{Cora}
&10\% &0.6508 &0.6699 &0.6678 &0.5981 &0.6641 &0.7149 &0.7978 &0.7996 &0.6436 &0.7068 &\underline{0.7569} &0.7513 &0.8187 &0.7617 &\textbf{0.8197} \\
&30\% &0.7214 &0.7908 &0.7220 &0.7254 &0.7449 &0.7939 &0.8245 &0.8350 &0.6696 &0.7347 &\underline{0.8184} &0.8147 &0.8524 &0.8208 &\textbf{0.8561}  \\
&50\% &0.7589 &0.8164 &0.7373 &0.7487 &0.7705 &0.8135 &0.8361 &0.8366 &0.6786 &0.7607 &\underline{0.8426} &0.8372 &0.8550 &0.8331 &\textbf{0.8615}  \\
\hline
\multirow{3}*{Wiki}
&10\% &0.1741 &0.3570 &0.5586 &0.4319 &0.4920 &0.5582 &0.5899 &0.5423 &0.6616 &0.6189 &0.5803 &\underline{0.5822} &0.7028 &0.7006 &\textbf{0.7180} \\
&30\% &0.2212 &0.5579 &0.6170 &0.5658 &0.5846 &0.6224 &0.6669 &0.6005 &0.6952 &0.6463 &0.6477 &\underline{0.6493} &0.7363 &0.7534 &\textbf{0.7580} \\
&50\% &0.2616 &0.6303 &0.6434 &0.5838 &0.6158 &0.6419 &0.6845 &0.6274 &0.7033 &0.6578 &\underline{0.6772} &0.6751 &0.7457 &0.7704 &\textbf{0.7711} \\
\hline
\multirow{3}*{PPI}
&10\% &0.0980 &0.1447 &0.1391 &0.0306 &- &- &0.1379 &0.1433 &0.0469 &0.0439 &- &- &0.1518 &\underline{0.1537} &\textbf{0.1659} \\
&30\% &0.1390 &0.1799 &0.1693 &0.0626 &- &- &0.1724 &0.1671 &0.0449 &0.0458 &- &- &\underline{0.1873} &0.1773 &\textbf{0.1956} \\
&50\% &0.1660 &0.1833 &0.1816 &0.0891 &- &- &0.1809 &0.1715 &0.0438 &0.0410 &- &- &\underline{0.1960} &0.1834 &\textbf{0.2065} \\
\hline
\multirow{3}*{Blogcatalog}
&10\% & 0.2683	&0.3192	&0.3311	&0.1632	&-	&-	&0.3302	&0.2371	&0.0271	&0.1121	&-	&-	&\underline{0.3372}	&0.3076	&\textbf{0.3399} \\
&30\% &0.2984	&0.3436	&0.3504	&0.2357	&-	&-	&0.3409	&0.2654	&0.0316	&0.1364	&-	&-	&\underline{0.3571}	&0.3261	&\textbf{0.3627} \\
&50\% &0.3249	&0.3485	&0.3600	&0.2803	&-	&-	&0.3431	&0.2741	&0.0492	&0.1365	&-	&-	&\underline{0.3621}	&0.3321	&\textbf{0.3692} \\
\hline
\end{tabular}
\label{tab_micro-f1}
\end{table*}

\begin{table*}[!t]
\setlength{\tabcolsep}{4pt} 
\scriptsize
\caption{Macro-F1 scores on classification tasks. The bold mark and underline mark have the same meanings as in Table~\ref{tab_micro-f1}.}
\centering
\begin{tabular}{cc||l|ll|lll|ll|ll||ll|l|ll}
\hline
\multicolumn{2}{c||}{\emph{\textbf{Information}}} & \multicolumn{1}{c|}{$X$} & \multicolumn{2}{c|}{$A$} & \multicolumn{3}{c|}{$A$, $\mathcal{L}$} &\multicolumn{2}{c|}{$A$, $X$} & \multicolumn{2}{c|}{$A$, $X$, $\mathcal{L}$} &  \multicolumn{2}{c|}{$A$, $\mathcal{L}$} &$A$, $X$ & \multicolumn{2}{c}{$A$, $X$, $\mathcal{L}$}\\
\hline
\multicolumn{2}{c||}{\backslashbox{\emph{\textbf{Data}}}{\emph{\textbf{Method}}}} & NodeFeats & MFDW & LINE & LSHM & LDE  &MMDW & TADW &DGI &GCN &APPNP & RSDNE & RSDNE* & \mydeepalg-N & \mydeepalg-L & \mydeepalg\\
\hline
\hline
\multirow{3}*{Citeseer}
&10\% &0.5860 &0.4195 &0.3856 &0.3724 &0.4155 &0.4707 &0.6294 &0.6310 &0.4761 &0.5175 &0.4949 &\underline{0.4994} &0.6233 &0.6089 &\textbf{0.6541} \\
&30\% &0.6504 &0.5253 &0.4315 &0.4990 &0.5030 &0.5606 &0.6679 &0.6382 &0.5085 &0.5567 &\underline{0.5939} &0.5852 &0.6669 &0.6647 &\textbf{0.6919}  \\
&50\% &0.6692 &0.5559 &0.4403 &0.5454 &0.5540 &0.5835 &0.6788 &0.6595 &0.5159 &0.5665 &\underline{0.6385} &0.6285 &0.6798 &0.6799 &\textbf{0.7016}  \\
\hline
\multirow{3}*{Cora}
&10\% &0.6182 &0.6598 &0.6478 &0.5595 &0.6453 &0.7043 &0.7823 &0.7540 &0.5623 &0.6308 &\underline{0.7436} &0.7367 &0.8084	&0.7457	&\textbf{0.8094} \\
&30\% &0.7103 &0.7819 &0.7099 &0.6625 &0.7343 &0.7830 &0.8127 &0.8257 &0.5856 &0.6641 &\underline{0.8073} &0.8029 &0.8438 &0.8008 &\textbf{0.8462}  \\
&50\% &0.7430 &0.8081 &0.7284 &0.6798 &0.7628 &0.8045 &0.8245 &0.8277 &0.5991 &0.6937 &\underline{0.8318} &0.8267 &0.8455 &0.8196 &\textbf{0.8502}  \\
\hline
\multirow{3}*{Wiki}
&10\% &0.0538 &0.2835 &0.4025 &0.3099 &0.3872 &0.4190 &0.4538 &0.3615 &0.4680 &0.4031 &\underline{0.4518} &0.4468 &0.5405 &0.5525 &\textbf{0.5789} \\
&30\% &0.1110 &0.4333 &0.4738 &0.3869 &0.4641 &0.4973 &0.5651 &0.4270 &0.4939 &0.4365 &0.5326 &\underline{0.5363} &0.6093 &0.6206 &\textbf{0.6480} \\
&50\% &0.1530 &0.4958 &0.5136 &0.4209 &0.5047 &0.5257 &0.6208 &0.4387 &0.4954 &0.4486 &\underline{0.5741} &0.5655 &0.6340 &0.6490 &\textbf{0.6573} \\
\hline
\multirow{3}*{PPI}
&10\% &0.0574 &0.0915 &0.0854 &0.0148 &- &- &0.0851 &0.0833 &0.0153 &0.0156 &- &- &0.0966 &\underline{0.1133} &\textbf{0.1191} \\
&30\% &0.0902 &0.1204 &0.1040 &0.0316 &- &- &0.1102 &0.0980 &0.0156 &0.0189 &- &- &\underline{0.1262} &0.1238 &\textbf{0.1402} \\
&50\% &0.1083 &0.1205 &0.1222 &0.0522 &- &- &0.1183 &0.1070 &0.0141 &0.0176 &- &- &\underline{0.1327} &0.1248 &\textbf{0.1491} \\
\hline
\multirow{3}*{Blogcatalog}
&10\% &0.1008	&0.1488	&0.1472	&0.0385	&-	&-	&0.1438	&0.0794	&0.0131	&0.0281	&-	&-	&\underline{0.1596}	&0.1187	&\textbf{0.1622} \\
&30\% &0.1157	&0.1721	&0.1727	&0.0894	&-	&-	&0.1571	&0.1042	&0.0139	&0.0299	&-	&-	&\underline{0.1887}	&0.1335	&\textbf{0.1921} \\
&50\% &0.1369	&0.1787	&0.1806	&0.1285	&-	&-	&0.1584	&0.1166	&0.0151	&0.0293	&-	&-	&\underline{0.1974}	&0.1396	&\textbf{0.1997} \\
\hline
\end{tabular}
\label{tab_macro-f1}
\end{table*}

\begin{enumerate}
  \item NodeFeats is a content-only baseline which only uses the original node features.
  \item MFDW~\cite{yang2015network} is the matrix factorization form of DeepWalk~\cite{perozzi2014deepwalk}. This method is unsupervised.
  \item LINE~\cite{tang2015line} is also a popular unsupervised method which considers the first-order and second-order proximity information.
  \item LSHM~\cite{jacob2014learning} is a semi-supervised method which considers the first-order proximity of a graph and jointly learns a linear classification model.
  \item LDE~\cite{wang2016linked} is a semi-supervised method which also considers the first-order proximity and jointly trains a 1-nearest neighbor classification model.
  \item MMDW~\cite{tu2016max} is a semi-supervised method which adopts MFDW model to preserve the graph structure and jointly trains an SVM model.
  \item TADW~\cite{yang2015network} is an unsupervised method which incorporates DeepWalk and associated node features
into the matrix factorization framework.
  \item DGI~\cite{velivckovic2019deep} is a recently proposed unsupervised GNN method which trains a graph convolutional encoder through maximizing mutual information.
  \item GCN~\cite{kipf2017semi} is the most well-known GNN method. This method is supervised.
  \item APPNP~\cite{klicpera2019predict} extends GCN with the idea of PageRank to explore the global graph structure. This method is also supervised.
\end{enumerate}

\mysubsection*{Parameters}
Following~\cite{tu2016max}, the embedding dimension is set to 200.
In addition, for DeepWalk, we adopt the default parameter setting i.e., window size is 5, walks per vertex is 80.
For LINE, we first learn two 100-dimension embeddings by adopting its first-order proximity and second-order proximity separately, and then concatenate them as suggested in~\cite{tang2015line}.
To fully show the limitations of those semi-supervised methods, we also tune their parameters by a “grid-search” strategy from $\{10^{-2},10^{-1},10^{0},10^{1},10^{2}\}$ and report the best results.
For these three GNNs (DGI, GCN and APPNP), we all use the code provided by the authors and adopt the default hyper-parameters.
As GCN and APPNP are end-to-end node classification methods, we use the outputs of their hidden layer (whose hidden units number is set to 200) as embedding results.
Additionally, as the original implementations of GCN and APPNP do not support multi-label tasks, we replace their loss functions by Binary Cross-entropy loss on PPI and Blogcatalog datasets.

In contrast, in RSDNE and its light version RSDNE*, we fix parameters $\alpha {=} 1$ and $\lambda {=} 0.1$ throughout the experiment.
In addition, we simply set the intra-class neighbor number $k{=}5$ like most manifold learning methods~\cite{zhu2006semi}, and set the candidate number $\bar{k}{=}20k$ for \myalg$^{*}$.

The settings of our \mydeepalg\ method and its two sub-methods are as follows.
We use \mydeepalg-L to denote the sub-method with the semantic preserving loss (i.e., Eq.~\ref{eq_semantic_loss}), and we use \mydeepalg-N to denote the sub-method with the graph preserving loss (i.e., Eq.~\ref{eq_deep_graph_loss}).
In \mydeepalg-L, we train a simple model with one GCN layer and one FC layer.
In addition, we use a simple averaging function as its readout function $\mathcal{R}$; and we apply SVD decomposition on the original node features to get 200-dimensional node features, for the calculation of semantic preserving loss.
Like the compared GNN baselines, we also use the outputs of our hidden layer in \mydeepalg-L as embedding results.
In \mydeepalg-N, we train a simple model with only one GCN layer.
In both sub-methods, we use the PReLU activation~\cite{he2015delving}, mean squared error (MSE) loss, and Xavier initialization~\cite{glorot2010understanding}.
We train all models for 100 epochs (training iterations) using Adam SGD optimizer~\cite{kingma2014adam} with a learning rate of 0.001.
Unless otherwise noted, all these settings are used throughout the experiments.

\subsection{Test with Completely-imbalanced Label}\label{subsect_exp_1}
\mysubsection*{Experimental setting}
Following~\cite{perozzi2014deepwalk}, we validate the quality of learned representations on node classification task.
As this study focuses on the completely-imbalanced label setting, we need to perform seen/unseen class split and remove the unseen classes from the training data.
Particularly, for Citeseer and Cora, we use two classes as unseen.
Thus, we have $C_{6}^{2}$ and $C_{7}^{2}$ different seen/unseen splits for Citeseer and Cora, respectively.
As Wiki, PPI and Blogcatalog contain much more classes, we randomly select five classes as unseen classes and repeat the split for 20 times.

The detailed experimental procedure is as follows.
First, we randomly sample some nodes as the training set (denoted as $\mathcal{L}$), and use the rest as the test set.
Then, we remove the unseen class nodes from $\mathcal{L}$ so as to obtain the completely-imbalanced labeled data $\mathcal{L}'$.
With the graph knowledge (i.e., $A$ and $X$) and $\mathcal{L}'$, we get the representations learned by various methods.
Note that no method can use the labeled data from unseen classes for embedding.
After that, we train a linear SVM classifier based on the learned representations and the original label information $\mathcal{L}$.
At last, the trained SVM classifier is evaluated on the test data.

\begin{figure*}[!t]\centering
\subfigure{\includegraphics[width=0.21\textwidth]{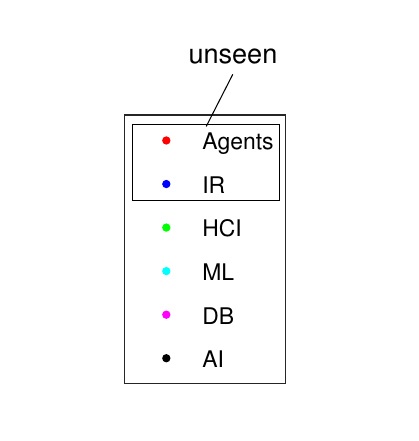}}
\addtocounter{subfigure}{-1}
\subfigure[LSHM]{    \includegraphics[width=0.235\textwidth]{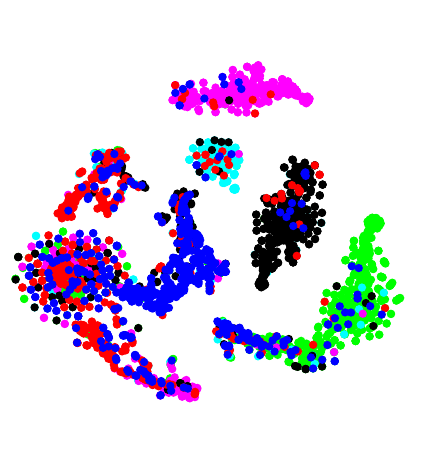}}
\subfigure[LDE]{    \includegraphics[width=0.235\textwidth]{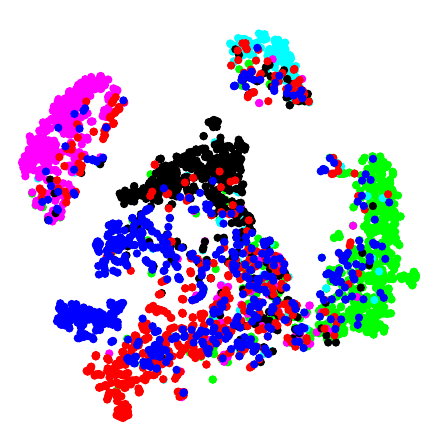}}
\subfigure[MMDW]{    \includegraphics[width=0.235\textwidth]{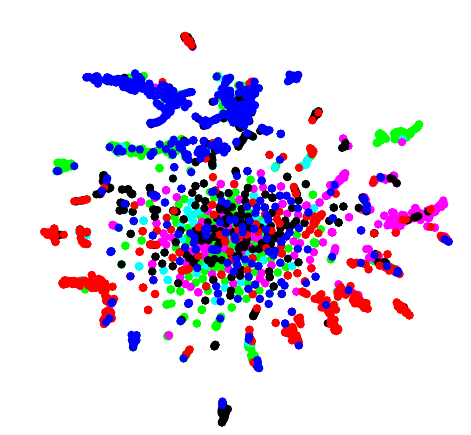}}
\subfigure[GCN]{    \includegraphics[width=0.235\textwidth]{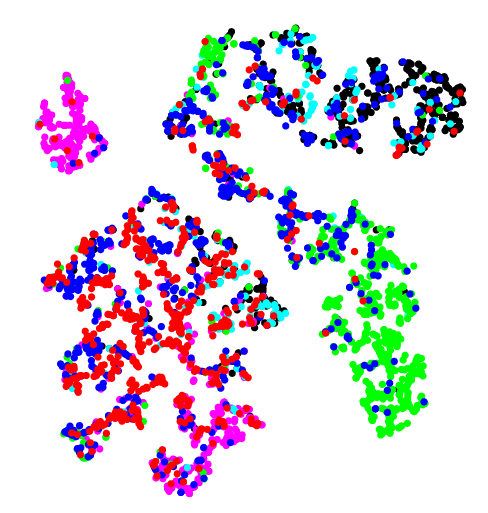}}
\subfigure[APPNP]{    \includegraphics[width=0.235\textwidth]{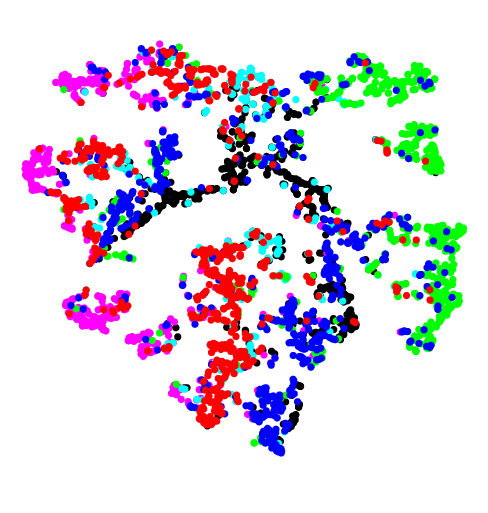}}
\subfigure[RSDNE]{    \includegraphics[width=0.235\textwidth]{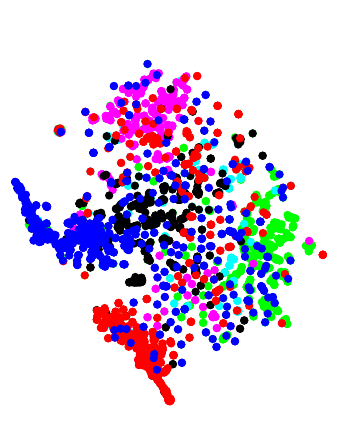}}
\subfigure[RSDNE$^*$]{    \includegraphics[width=0.235\textwidth]{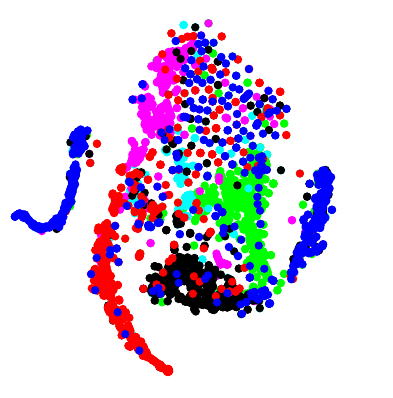}}
\subfigure[\mydeepalg-N]{    \includegraphics[width=0.235\textwidth]{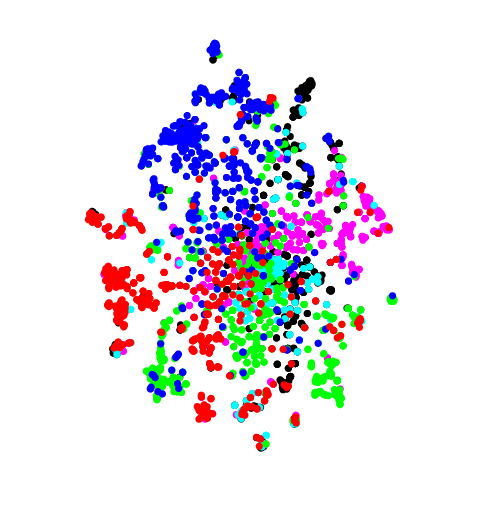}}
\subfigure[\mydeepalg-L]{    \includegraphics[width=0.235\textwidth]{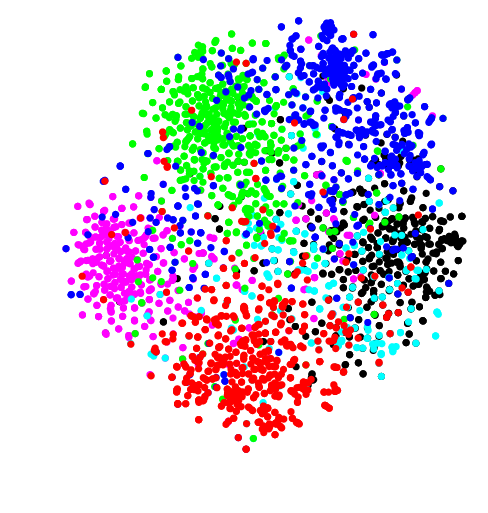}}
\subfigure[\mydeepalg]{    \includegraphics[width=0.235\textwidth]{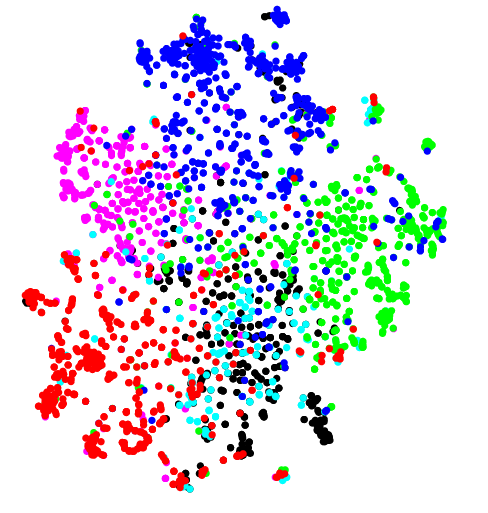}}
\caption{
2D visualization on Citeseer (50\% label rate with two unseen classes, i.e.,\{Agents, IR\}).}
\label{fig_2d}
\end{figure*}

\subsubsection{Node Classification Performance}
We vary the percentage of labeled data in [10\%, 30\%, 50\%] and then use the labeled nodes of seen classes as supervision for graph embedding learning.
We employ two widely used classification evaluation metrics: Micro-F1 and Macro-F1~\cite{yang1999evaluation}.
In particular, Micro-F1 is a weighted average of F1-scores over different classes, while Macro-F1 is an arithmetic mean of
F1-scores on each label:
\begin{equation}
\label{eq_micro_macro}
\begin{aligned}
\mathrm{Micro} \small{-} \mathrm{F1} & = \frac{\sum_{i=1}^{|\mathcal{C}|} 2\mathit{TP}^i}{\sum_{i=1}^{|\mathcal{C}|}(2\mathit{TP}^i+\mathit{FP}^i+\mathit{FN}^i)} \\
\mathrm{Macro} \small{-} \mathrm{F1} & = \frac{1}{|\mathcal{C}|} \sum_{i=1}^{|\mathcal{C}|} \frac{2\mathit{TP}^i}{(2\mathit{TP}^i+\mathit{FP}^i+\mathit{FN}^i)}
\end{aligned}
\end{equation}%
where $|\mathcal{C}|$ is the class number, $\mathit{TP}^i$ denotes the number of positives in the $i$-th class, $\mathit{FP}^i$ and $\mathit{FN}^i$ denotes the number of false positives and false negatives in the $i$-th class, respectively.

The results are presented in Tables~\ref{tab_micro-f1} and~\ref{tab_macro-f1}, from which we have the following observations\footnote{We do not test LDE, MMDW, RSDNE, and RSDNE* on PPI and Blogcatalog, since these methods could not handle the multi-label case.
Some experiments tested on more label rates can be found in~\cite{wang2018rsdne}.
}.

Firstly, our deep method \mydeepalg\ always achieves the best results on all datasets including both single-label and multi-label graphs.
This can be explained by the performance of \mydeepalg-L.
We can clearly find that \mydeepalg-L always outperforms the compared semi-supervised GNNs (i.e., GCN and APPNP) by a large margin (around 20\% to 300\% relatively).
This indicates that, by exploring the class-semantic knowledge, \mydeepalg\ can effectively utilize the completely-imbalanced labels.

Secondly, our shallow method \myalg\ and its light version both perform much better than all baselines which do not use node attributes.
For example, with 50\% labeled data, our two methods outperform the best baseline MMDW by 7--12\% relatively in term of Micro-F1.
The underlying principle is that our approximation models (i.e., Eq.~\ref{eq_intra_cost} and Eq.~\ref{eq_inter_cost}) reasonably guarantee both intra-class similarity and inter-class dissimilarity, and meanwhile avoids misleading the jointly trained graph structure preserving model.
Besides, the light version of our method \myalg$^{*}$ is competitive with \myalg.
This means that we can reduce the intra-class neighbor candidate number to make our method more efficient.

Thirdly, our deep method \mydeepalg\ is more powerful than our shallow method \myalg.
For example, in Citeseer with 30\% labeled data, \mydeepalg\ outperforms \myalg\ by 17.46\% relatively in term of Micro-F1.
The reason mainly lies in two folds.
On the one hand, benefiting from the powerful GNN layers, \mydeepalg\ could utilize the attributes of nodes.
On the other hand, exploring the knowledge of class-semantic descriptions (via a simple readout function) enables \mydeepalg\ to handle multi-label setting.

Lastly, all compared semi-supervised baselines become ineffective, and some of them even perform worse than unsupervised ones.
For example, LSHM and LDE achieve lower accuracy than MFDW in most cases;
GCN and APPNP also perform worse than DGI almost all the time.
This is consistent with our theoretical analysis (Section~\ref{sect_compare}) that traditional semi-supervised methods could get unappealing results in this completely-imbalanced label setting.

\subsubsection{Graph Layouts}
Following~\cite{tang2015line}, we use t-SNE package~\cite{maaten2008visualizing} to map the learned representations of Citeseer into a 2D space.
Without loss of generality, we simply adopt Citeseer's first two classes as unseen classes, and set the training rate to 50\%.
(Due to space limitation, we only visualize the embeddings obtained by semi-supervised methods.)

First of all, the visualizations of our GNN methods (\mydeepalg\ and its sub-methods), as expected, exhibit the most discernible clustering.
Especially, as shown in Fig.~\ref{fig_2d}(i), \mydeepalg-L which utilizes label information successfully respects the six topic classes of Citeseer.
In this visualization, we also note that the clusters of different classes do not separate each other by a large margin.
This is consistent with our analysis that the class-semantic preservation can be seen as a relaxation of the classical classification loss.
Additionally, as shown in Fig.~\ref{fig_2d}(j), \mydeepalg\ obtains the best visualization result, in which different topic classes are clearly separated.

Additionally, the visualizations of our RSDNE and RSDNE$^*$ are also quite clear, with meaningful layout for both seen and unseen classes.
As shown in Figs.~\ref{fig_2d}(f-g), the nodes of the same class tend to lie on or close to the same manifold.
Notably, the nodes from two unseen classes avoid heavily mixing with the wrong nodes.
Another surprising observation is that: compared to RSDNE, the embedding results of its light version (i.e., RSDNE*) seem to lie on more compact manifolds.
The reason might be that RSDNE* has a stricter manifold constraint, i.e., a labeled node's $k$ intra-class neighbors are adaptively selected from a predetermined candidate set.
The similar observation can be found in traditional manifold learning methods~\cite{roweis2000nonlinear} in which the neighbor relationships among instances are predetermined.

In contrast, all the compared semi-supervised baselines get unappealing visualizations. For example, as shown in Figs.~\ref{fig_2d}(a-b), although LSHM and LDE better cluster and separate the nodes from different seen classes, their two kinds of unseen class nodes heavily mix together.
The similar observation can be found in the results of semi-supervised GNNs (i.e., GCN and APPNP), as shown in Figs.~\ref{fig_2d}(d-e).
In addition, as shown in Fig.~\ref{fig_2d}(c), MMDW also fails to benefit from the completely-imbalanced labels.
This is because MMDW has to use a very small weight for its classification model part to avoid poor performance.

\subsubsection{Effectiveness Verification}
In the following experiments, we only show the results on Citeseer, since we get similar results on the other datasets.

\begin{figure}[t]
\centering
\subfigure{
    \includegraphics[width=0.22\textwidth]{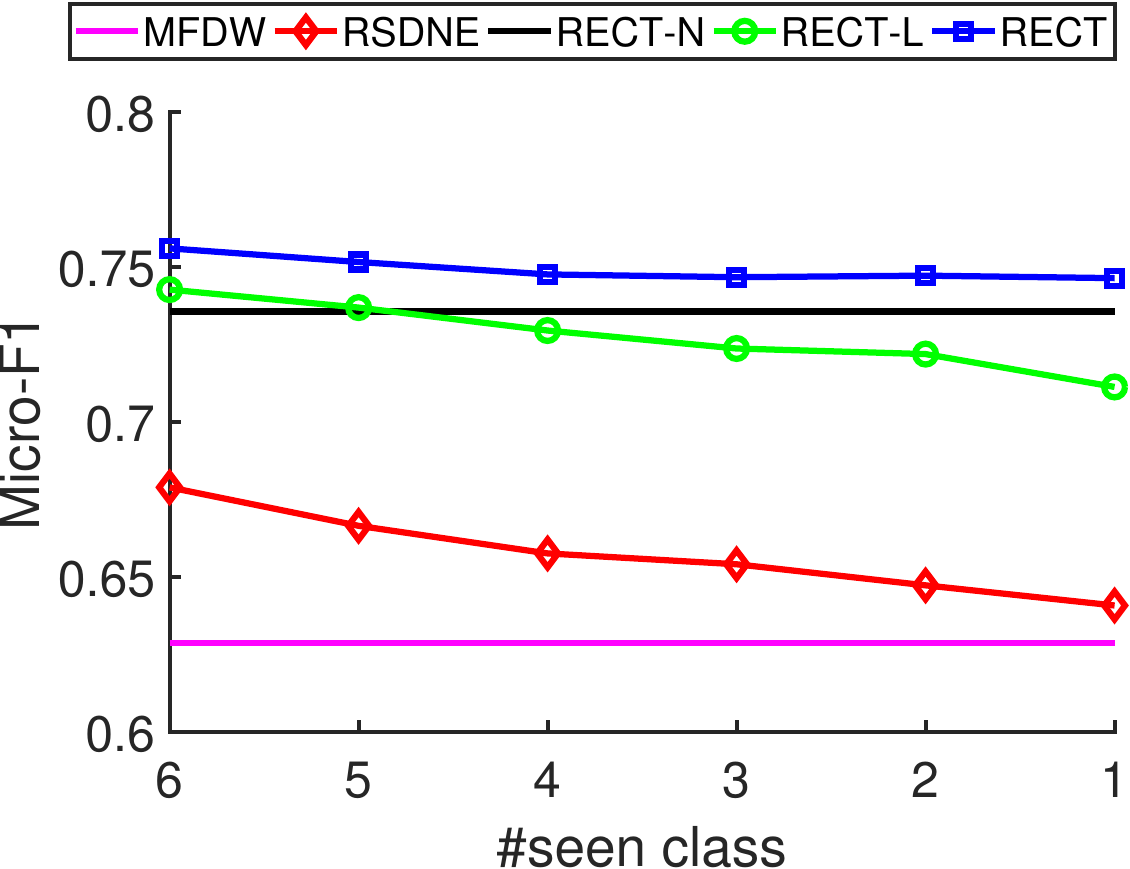}
}
\subfigure{
    \includegraphics[width=0.22\textwidth]{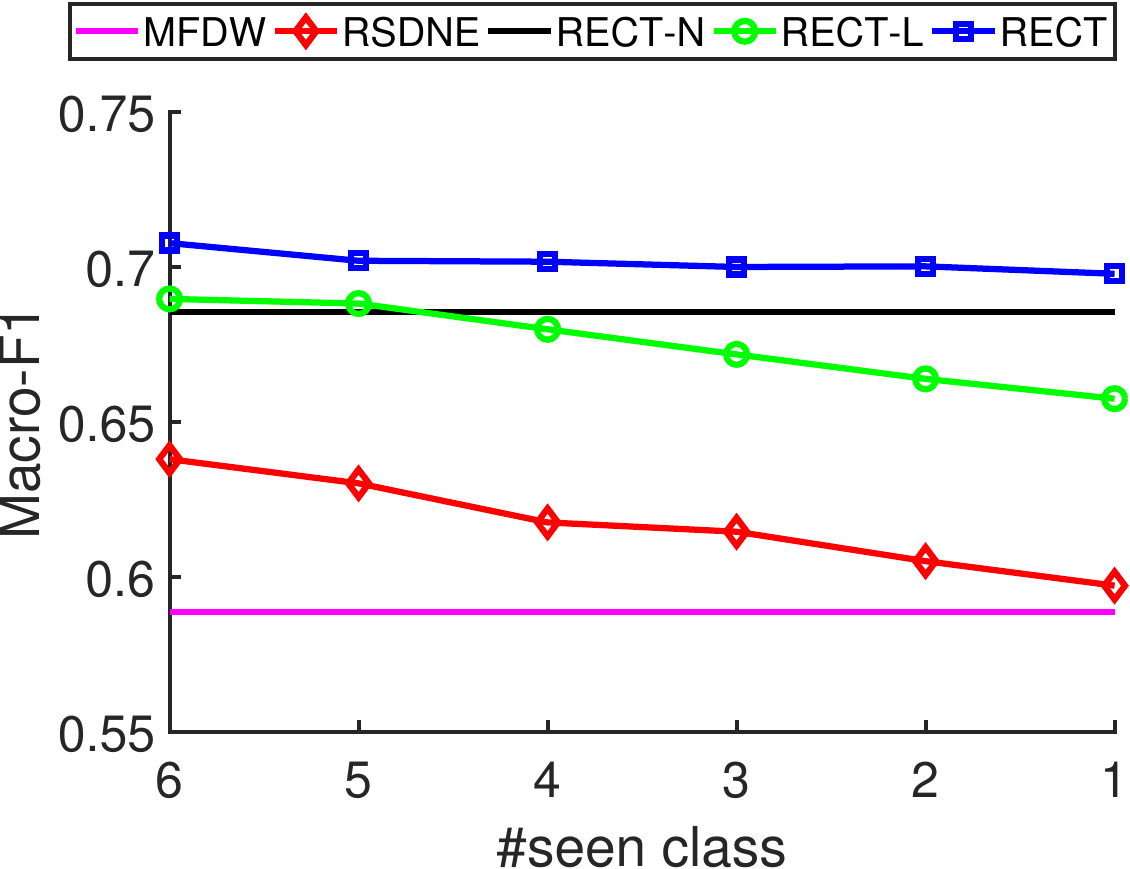}
}
\caption{Node classification performance w.r.t. the seen class number on Citeseer (with 50\% label rate).}
\label{fig_remove_number}
\end{figure}
\begin{figure}[!t]
\centering
\subfigure{
    \includegraphics[width=0.22\textwidth]{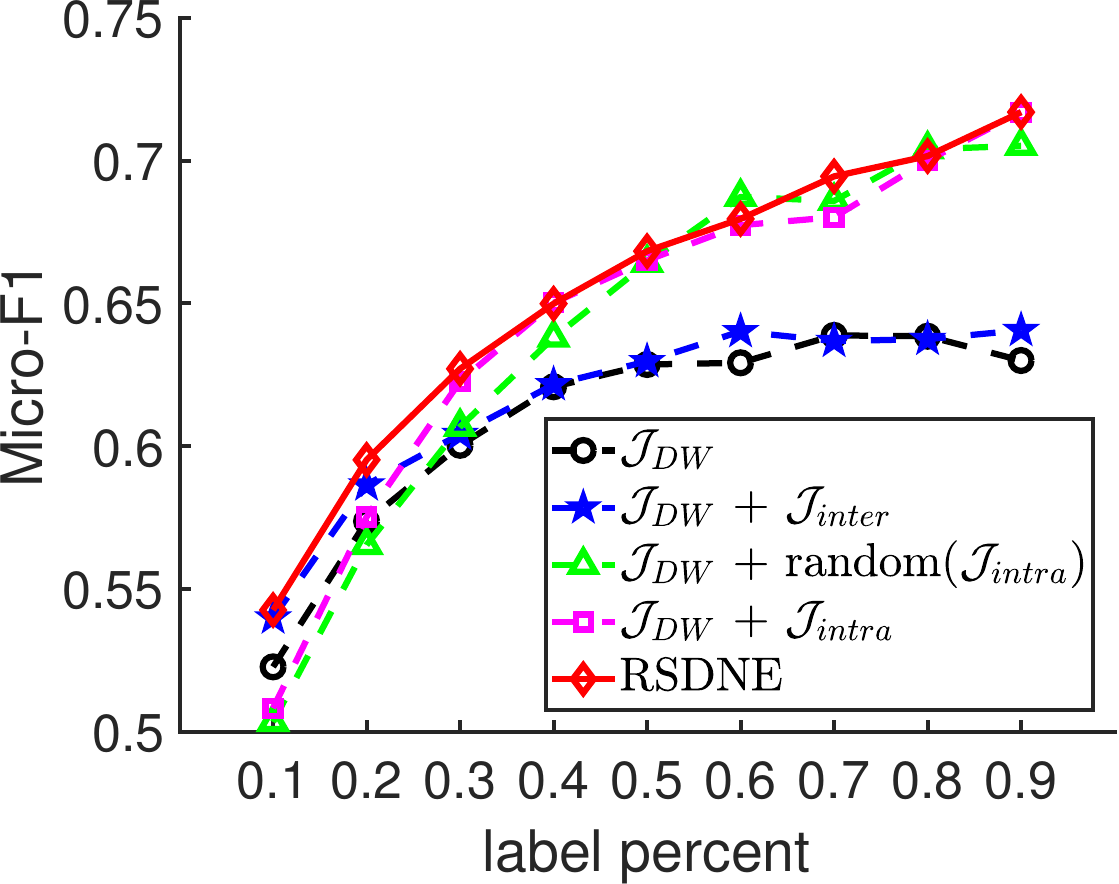}
}
\subfigure{
    \includegraphics[width=0.22\textwidth]{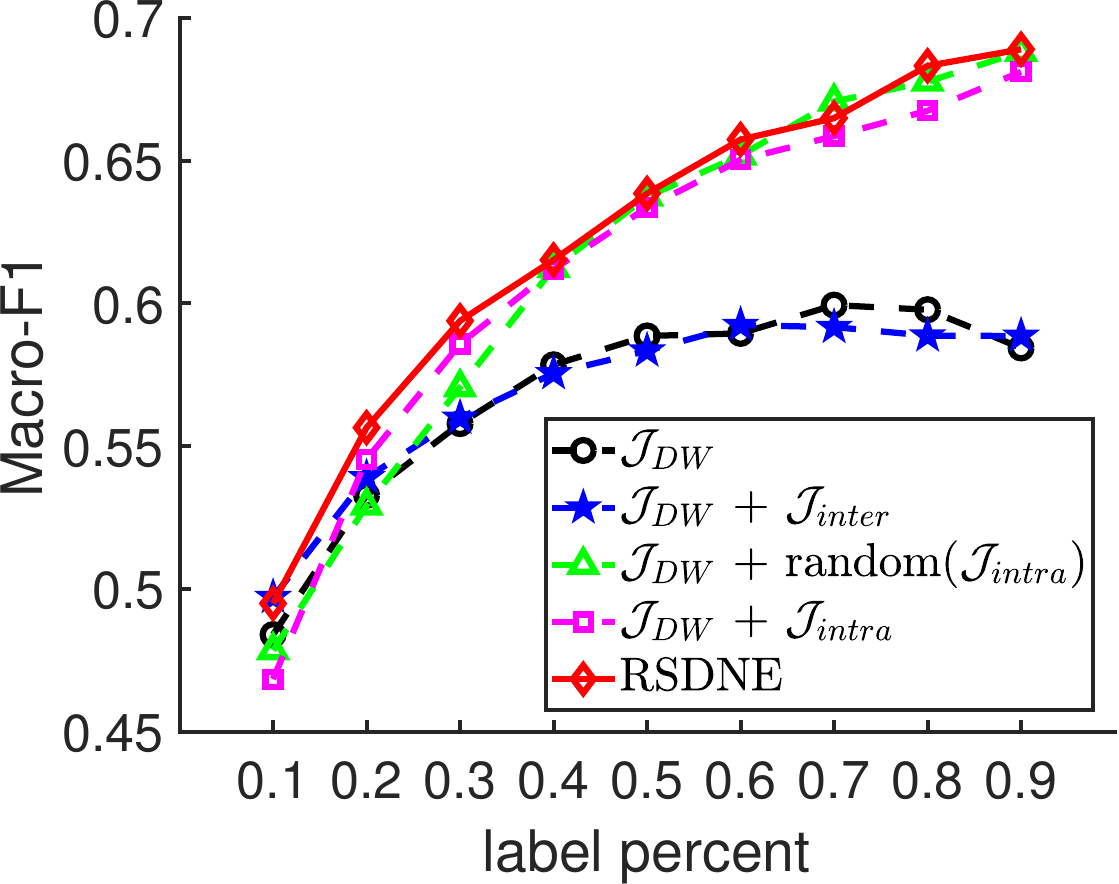}
}
\caption{Node classification performance w.r.t. different settings of \myalg\ on Citeseer.}
\label{fig_cite_function}
\end{figure}

\begin{figure}[t!]
\centering
\subfigure{
    \includegraphics[width=0.22\textwidth]{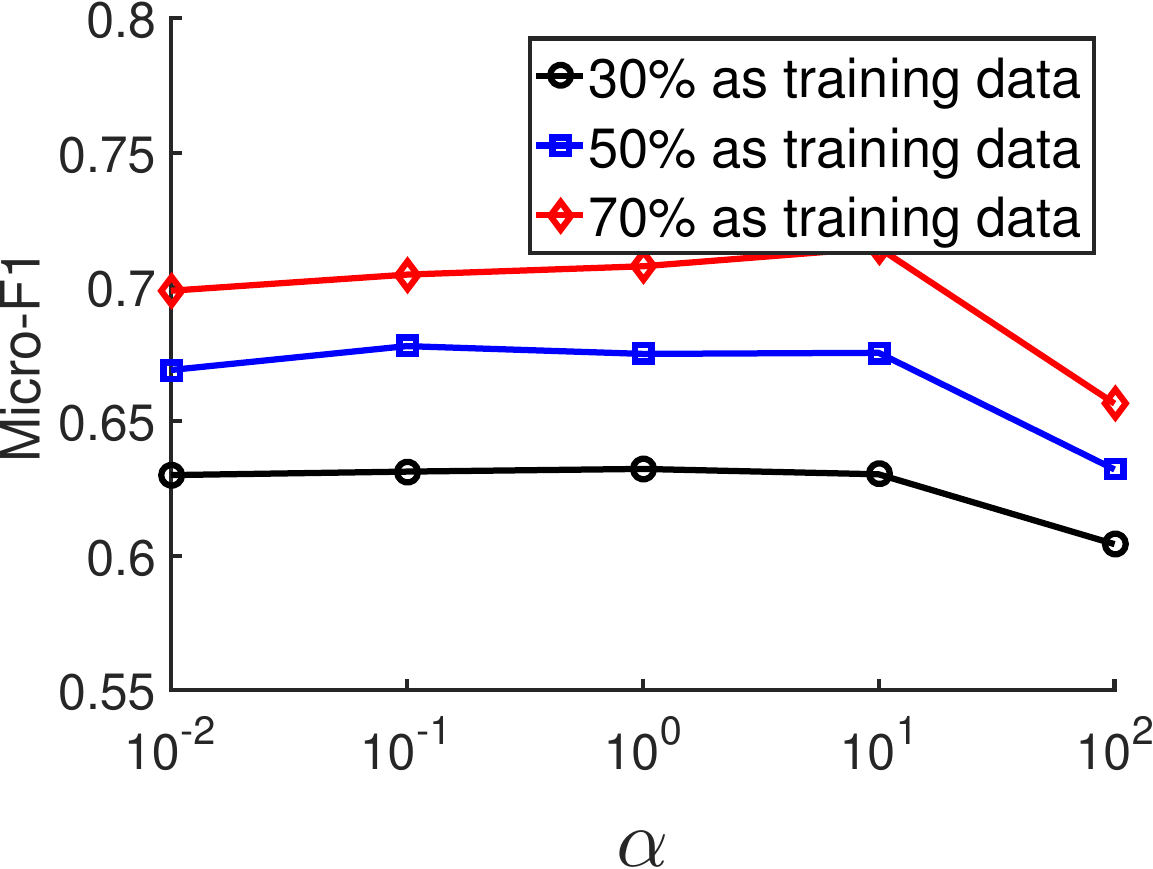}
}
\subfigure{
    \includegraphics[width=0.22\textwidth]{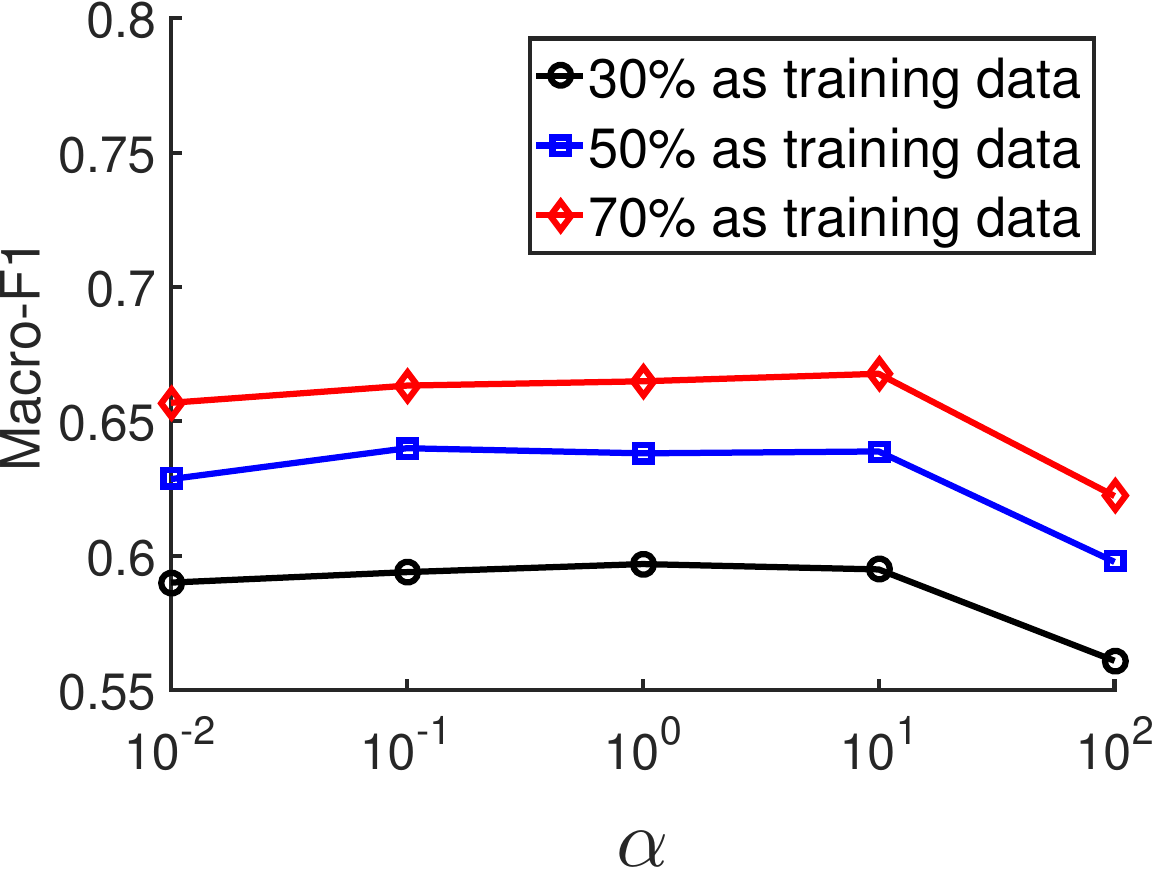}
}
\caption{The effect of parameter $\alpha$ in RSDNE on Citeseer.}
\label{fig_para}
\end{figure}

\begin{figure}[t!]
\centering
\subfigure{
    \includegraphics[width=0.22\textwidth]{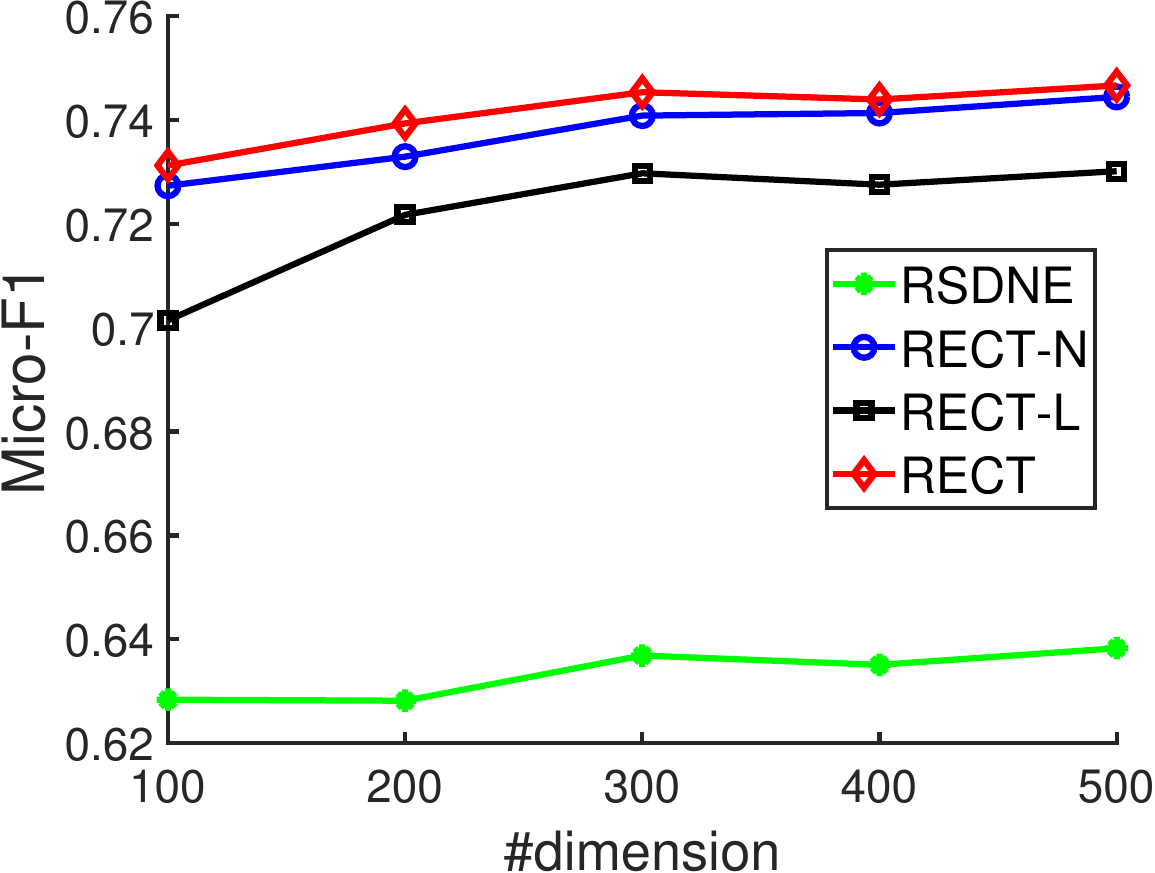}
}
\subfigure{
    \includegraphics[width=0.22\textwidth]{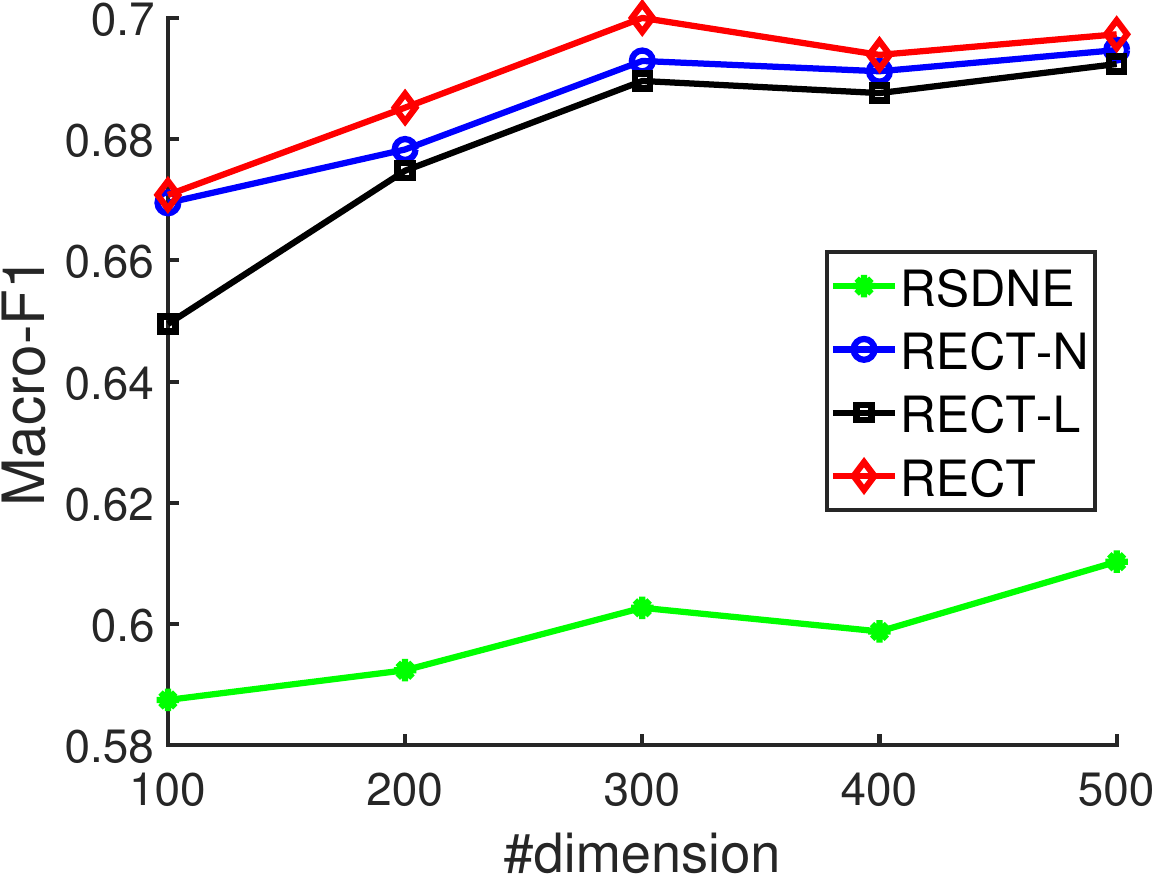}
}
\caption{The effect of embedding dimension on Citeseer with label rate 30\%.}
\label{fig_effect_dim}
\end{figure}

\begin{figure*}[ht]
\centering
    \includegraphics[width=1.0\textwidth]{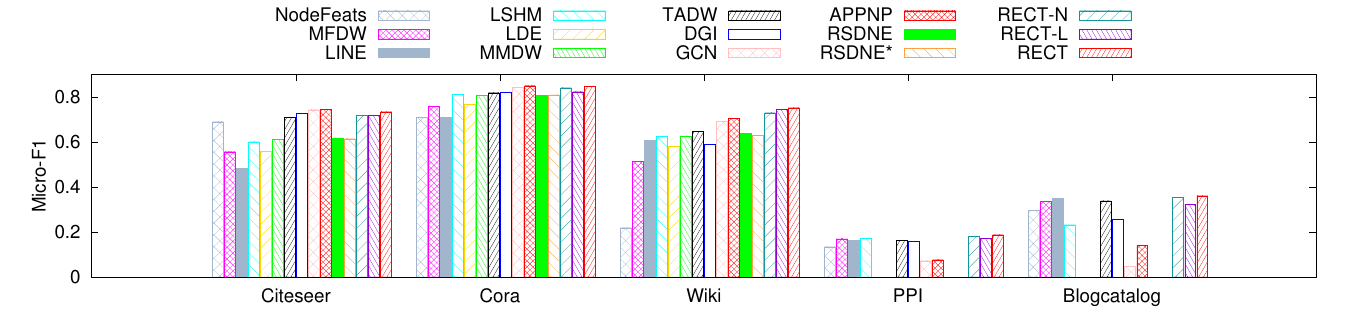}
\caption{Averaged node classification performance (Micro-F1) with balanced labels.}
\label{fig_balanced}
\end{figure*}

\mysubsection*{Effect of Seen/Unseen Class Number}
Without loss of generality, we set the training rate to 50\%, and vary the seen class number from six to one on Citeseer.
As shown in Fig.~\ref{fig_remove_number}, \myalg\ and \mydeepalg\ can constantly benefit from the completely-imbalanced labels.
For example, even with only one seen class, \myalg\ still outperforms (its unsupervised version) MFDW; \mydeepalg\ still outperforms (its unsupervised version) \mydeepalg-N.
Besides, the performance of \mydeepalg-L declines smoothly when the unseen class number grows, clearly demonstrating the effectiveness of exploring the class-semantic knowledge for the studied problem.

\mysubsection*{Effect of Intra-class Similarity and Inter-class Dissimilarity Modeling in RSDNE}
To investigate the effect of these two parts, we test the following settings of \myalg:
\begin{enumerate}
  \item $\mathcal{J}_{DW}$: only modeling the graph structure (Eq.~\ref{eq_mf_basic}).
  \item $\mathcal{J}_{DW} {+} \mathcal{J}_{intra}$: modeling graph structure and intra-class similarity (selecting intra-class neighbors adaptively (Eq.~\ref{eq_intra_cost})).
  \item $\mathcal{J}_{DW} + $ random($\mathcal{J}_{intra}$): modeling graph structure and intra-class similarity (selecting intra-class neighbors randomly).
  \item $\mathcal{J}_{DW} {+} \mathcal{J}_{inter}$: modeling graph structure and inter-class dissimilarity (Eq.~\ref{eq_inter_cost}).
\end{enumerate}

As shown in Fig.~\ref{fig_cite_function}, when either eliminating the effect of intra-class or inter-class modeling part, the performance degrades.
This suggests that these two parts contain complementary information to each other for graph embedding.
Another interesting observation is that: although randomly selecting intra-class neighbors (i.e., $\mathcal{J}_{DW} + $ random($\mathcal{J}_{intra}$)) does  not show the best result, it still outperforms modeling graph structure alone (i.e., $\mathcal{J}_{DW}$) significantly, especially when  the labeled dataset becomes larger.
This again shows the effectiveness of modeling the (relaxed) intra-class similarity.

\subsubsection{Sensitivity Analysis}
\mysubsection*{Sensitivity of Parameter}
In the proposed method \myalg, there is an important parameter $\alpha$ which balances the contributions of graph structure and label information.
Figure~\ref{fig_para} shows the classification performance with respect to this parameter on Citeseer (with the regularization parameter $\lambda{=}0.1$).
It can be observed that our method is not sensitive to $\alpha$ especially when $\alpha \in [10^{-2}, ... , 10^{1}]$.

\mysubsection*{Sensitivity of Embedding Dimension}
We vary embedding dimensions in \{100, 200, 300, 400, 500\}.
As shown in Fig.~\ref{fig_effect_dim}, all our methods are not very sensitive to the embedding dimension.
In addition, we can find that \mydeepalg\ always outperforms its two sub-methods \mydeepalg-N and \mydeepalg-L.
Another observation needs to be noted is that \mydeepalg\ still outperforms all baselines when the embedding dimension is set to 200.
All these observations demonstrate the superiority of our methods.

\subsection{Test with Balanced Labels}
We also test the situation where the labeled data is generally balanced, i.e., the labeled data covers all classes.
Figure~\ref{fig_balanced} shows the averaged classification performance (training ratio also varies in [10\%, 30\%, 50\%]). We can get the following two interesting observations.

The first and the most interesting observation is that our methods have comparable performance to state-of-the-art semi-supervised methods, although our methods are not specially designed for this balanced case.
Specifically, RSDNE and RSDNE$^{*}$ obtain comparable performance to LSHM, LDE and MMDW;
\mydeepalg\ obtains comparable (and sometimes much superior) results to GCN and APPNP.
This suggests that our methods would be favorably demanded by the scenario where the quality of the labeled data cannot be guaranteed.

The second observation is that our deep method \mydeepalg\ is more robust than the compared deep semi-supervised GNNs.
As shown in Fig.~\ref{fig_balanced}, GCN and APPNP perform poorly on two multi-label datasets PPI and Blogcatalog.
This may due to the imbalance of labels in these two datasets.
In contrast, our method RECT is much more stable on all datasets.
This might indicate that the distribution of class-semantic descriptions over various classes is more balanced than that of class labels.
All these observations show the general applicability of our approximation models (i.e., Eq.~\ref{eq_intra_cost}, Eq.~\ref{eq_inter_cost} and Eq.~\ref{eq_semantic_loss}) which could also be considered in other related applications.


\begin{figure}[t]
\centering
    \includegraphics[width=0.40\textwidth]{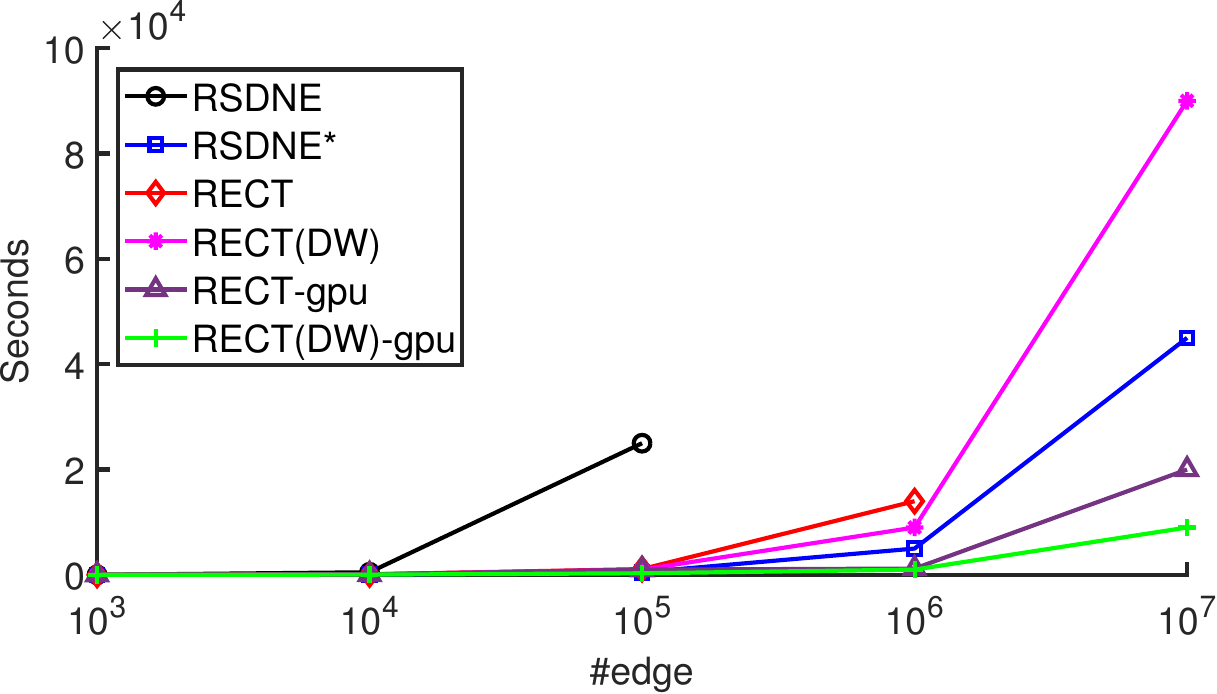}
\caption{
Training time of our methods. We do not report the running time when it exceeds 25 hours.
}
\label{fig_scalability}
\end{figure}

\subsection{Scalability Test}
Following~\cite{kipf2017semi}, we use random graphs to test the scalability.
Specifically, we create a random graph with $n$ nodes and $2n$ edges.
We take the identity matrix as the input feature matrix $X$.
We give same label for all nodes, set training rate to 10\% and do not remove any labeled nodes.

We test RSDNE, RSDNE$^{*}$ and together with different implements of our RECT method:
1) RECT is the original proposed GNN method;
2) RECT(DW) adopts the objective of DeepWalk for graph structure preserving (i.e., Eq.~\ref{eq_deepwalk_original});
3) RECT-gpu is the GPU implementation of RECT;
4) RECT(DW)-gpu is the GPU implementation of RECT(DW).
Our methods are written in Python 3.0 and Pytorch 1.0.
All the codes are running on a server with 16 CPU cores, 32 GB main memory, and an Nvidia Titan V GPU.
Figure~\ref{fig_scalability} shows the running times.
We can find that RSDNE$^{*}$ is more efficient than RSDNE, which is consistent with our theoretical analysis.
We also find that RECT(DW) is more efficient than RECT, indicating we can adopt various graph structure preserving objectives to accelerate our method.
In addition, the GPU implementation of GNN methods can largely accelerate the training speed.

\section{Conclusion} \label{section_conclusion}
This paper investigates the graph embedding problem in the completely-imbalanced label setting where the labeled data cannot cover all classes.
We firstly propose a shallow method named RSDNE.
Specifically, to benefit from completely-imbalanced labels, RSDNE guarantees both intra-class similarity and inter-class dissimilarity in an approximate way.
Then, to leverage the power of deep neural networks, we propose \mydeepalg, a new class of GNN.
Unlike RSDNE, \mydeepalg\ utilizes completely-imbalanced labels by exploring the class-semantic descriptions, which enables it to handle graphs with node features and multi-label setting.
Finally, extensive experiments are conducted on several real-world datasets to demonstrate the effectiveness of the proposed methods.
In the future, we plan to extend our methods to other types of graphs, such as heterogeneous graphs and signed graphs.
\bibliographystyle{IEEEtran}
\bibliography{simple}
\ifCLASSOPTIONcompsoc
  \section*{Acknowledgments}
\else
  \section*{Acknowledgment}
\fi

Zheng Wang is supported in part by the National Natural Science Foundation of China (No.~61902020) and the China Postdoctoral Science Foundation Funded Project (No.~2018M640066).
Chaokun Wang is supported in part by the National Natural Science Foundation of China (No.~61872207) and Baidu Inc.
Xiaojun Ye is supported by the National Key R\&D Program of China (No.~2019QY1402).
Philip S.~Yu is supported in part by NSF under grants (III-1526499, III-1763325, III-1909323, and CNS-1930941).

\ifCLASSOPTIONcaptionsoff
  \newpage
\fi

\end{document}